\documentclass[letterpaper,journal]{IEEEtran}

\usepackage{amsmath,amsfonts,amsthm}
\usepackage{bm}
\usepackage{mathrsfs}
\usepackage{microtype}
\usepackage{makecell}

\usepackage{graphicx}
\usepackage[table]{xcolor}
\usepackage{caption}

\usepackage{array}
\usepackage{booktabs}
\usepackage{multirow}
\usepackage{tabularx}

\usepackage{siunitx}
\usepackage{enumitem}
\usepackage{pifont}

\usepackage[ruled,linesnumbered,vlined]{algorithm2e}

\usepackage[caption=false,font=footnotesize]{subfig}

\usepackage{svg}
\usepackage{orcidlink}
\usepackage{newunicodechar}
\newunicodechar{✓}{\checkmark}
\usepackage[utf8]{inputenc} 

\usepackage{hyperref}
\hypersetup{
  colorlinks=true,
  linkcolor=blue,
  urlcolor=magenta,
  citecolor=blue
}

\usepackage[switch]{lineno}

\theoremstyle{remark}
\newtheorem{remark}{\indent Remark}

\newtheorem{theorem}{\indent Theorem}
\newtheorem{definition}{Definition}

\definecolor{darkred}{RGB}{192,0,0}

\begin{document}

\title{Hierarchical Reference Sets for Robust Unsupervised \\ Detection of Scattered and Clustered Outliers}

\author{Yiqun Zhang\orcidlink{0000-0002-0328-987X},~\IEEEmembership{Senior Member,~IEEE}, 
Zexi Tan\orcidlink{0009-0007-8919-9077},
Xiaopeng Luo\orcidlink{0009-0005-9214-7716}~\IEEEmembership{Member,~IEEE},
Yunlin Liu\orcidlink{0009-0004-3718-6790}
        % <-this % stops a space
\thanks{This work was supported in part by the National Natural Science Foundation of China (NSFC) under grant: 62476063, the Natural Science Foundation of Guangdong Province under grant: 2025A1515011293, and the General University Youth Innovation Talent
Program of Guangdong Province under grant 2024KQNCX094.}
\thanks{All authors are with the School of Computer Science and Technology, Guangdong University of Technology, Guangzhou, China. (\textit{Corresponding author: Xiaopeng Luo}, e-mail: luoxiaopeng@mails.gdut.edu.cn) 
}}

% The paper headers
\markboth{IEEE Internet of Things Journal}%
{Shell \MakeLowercase{\textit{et al.}}: A Sample Article Using IEEEtran.cls for IEEE Journals}

\maketitle
%\linenumbers
\begin{abstract}
Most real-world IoT data analysis tasks, such as clustering and anomaly event detection, are unsupervised and highly susceptible to the presence of outliers. In addition to sporadic scattered outliers caused by factors such as faulty sensor readings, IoT systems often exhibit clustered outliers. These occur when multiple devices or nodes produce similar anomalous measurements, for instance, owing to localized interference, emerging security threats, or regional false alarms, forming micro-clusters. These clustered outliers can be easily mistaken for normal behavior because of their relatively high local density, thereby obscuring the detection of both scattered and contextual anomalies. To address this, we propose a novel outlier detection paradigm that leverages the natural neighboring relationships using graph structures. This facilitates multi-perspective anomaly evaluation by incorporating reference sets at both local and global scales derived from the graph. Our approach enables the effective recognition of scattered outliers without interference from clustered anomalies, whereas the graph structure simultaneously helps reflect and isolate clustered outlier groups. Extensive experiments, including comparative performance analysis, ablation studies, validation on downstream clustering tasks, and evaluation of hyperparameter sensitivity, demonstrate the efficacy of the proposed method. The source code is available at \url{https://github.com/gordonlok/DROD}.
\end{abstract}

\section{Introduction}%label检查
\IEEEPARstart{O}{utlier} detection is crucial in many machine learning and data mining tasks, especially in Internet of Things (IoT) applications \cite{chandola2009anomaly}. It enables the discovery of emerging concepts~\cite{zhao2023unsupervised,zhao2022heterogeneous} in IoT scenarios, including novel attack behaviors and previously unseen device malfunctions. It additionally serves as a critical preprocessing stage for filtering abnormal sensor measurements, thus enhancing the reliability of downstream tasks such as device status monitoring~\cite{yang2021mean} and real-time decision automation \cite{zhang2025learningICD}. Since data labels are usually unavailable in most real IoT scenarios due to the massive scale of deployments and dynamic environments~\cite{zhang2025learning}, unsupervised outlier detection remains a challenging task for IoT data analytics~\cite{chen2024mgod}. From the perspective of the number of data instances involved in forming abnormal patterns, outliers are typically categorized into two categories: point outliers and collective outliers \cite{boukerche2020outlier}. A point outlier is a single data instance distinctly deviant from the rest of the dataset; it represents the most straightforward and well-studied type of outlier. In contrast, collective outliers consist of a group of instances that are anomalous as a whole, a phenomenon frequently encountered in various practical applications.

\begin{figure}[t]
	\centering
         \includegraphics[width=0.9\linewidth]{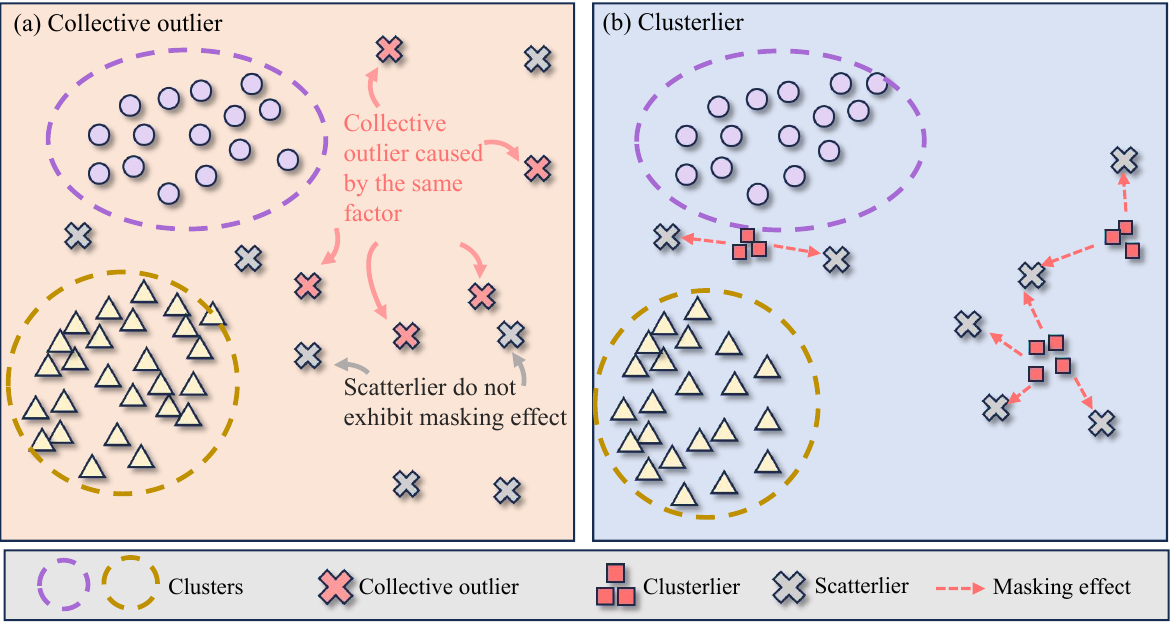}
\caption{Illustration of scatterliers (gray cross), clusterliers (red square), collective outliers (red cross) and the masking effect in outlier detection (red dotted arrow) of a dataset with normal clustered samples (dotted circle).}
	\label{figure-outliers}
\end{figure}

To better characterize the spatial manifestations of anomalies, this paper refines the established categories from a distribution perspective, i.e., scatterliers and clusterliers. A scatterlier represents a spatially isolated point outlier that significantly deviates from the majority of the data and lies in sparse regions away from dense clusters. In contrast, a clusterlier is a form of collective outlier that forms a compact but anomalous micro-cluster~\cite{cai2024robust,zhang2025adaptive}, where points are mutually close, creating a locally dense region, yet remain globally divergent from normal clusters. While both types are illustrated in Fig. \ref{figure-outliers}, clusterliers present a particularly challenging detection problem due to the masking effect caused by their spatial proximity: typical collective outliers and clusterliers are distinct IoT anomalies with fundamentally different traits and detectability. Collective outliers, referring to grouped samples, arise from a common cause. For example, multiple sensors in an area report abnormally high power use together due to a shared power outage \cite{fengmobile}. They have consistent, recognizable patterns, so conventional methods detect them easily. In contrast, clusterliers are densely grouped anomalous points. A typical case is a botnet of compromised devices communicating in similar malicious ways. Within such clusters, individual anomalies mask each other, making each malicious device seem normal relative to others and the whole cluster misclassified as legitimate. This mutual masking effect gradually accumulates and ultimately suppresses the detectability of clusterliers, rendering them much less observable in practice than collective outliers.

To detect various types of outliers, existing methods can be categorized into global and local methods according to the ``reference set'' they adopted for computing the anomaly indices of samples \cite{campos2016evaluation}. The reference set of a sample usually acts as the statistical basis for the anomaly score computation and thus usually dominates the detection results. Global methods refer to those that utilize the entire dataset as an identical reference set for each sample. Some typical global methods~\cite{Iforest, GAOXN2022} usually rely on implicit assumptions about the whole data distribution, e.g., a certain number of clusters with randomly distributed scatterliers, relatively large spherical-shaped clusters with very small clusterliers, etc. Once the real scenario does not match the predefined theoretical assumptions, or even when mixed or hybrid distributions occur, the detection accuracy of the existing global methods can be greatly influenced \cite{outlieroverview}.

By contrast, the local methods are more flexible in determining outliers based on local distribution and, thus, often introduce less bias and achieve better performance. $k$-Nearest Neighbors ($k$NN)-based methods \cite{lof,DGOF} detect outliers locally according to the $k$ nearest neighbors of each sample as the reference set. They adopt the common basic idea that if a sample is connected with fewer neighbors, it is determined as an outlier. In recent years, several advanced $k$NN-based methods \cite{tang2017local} have been proposed, employing various perspectives to compute the anomaly indices based on the $k$ nearest neighbors of the samples. To address the non-trivial determination of the number of neighbors $k$, some methods \cite{wahid2021nanod} have been proposed that adopt the Natural Neighbor concept \cite{Zhu2016}, which can adaptively select $k$ for different samples with considering their local distributions. However, all the above-mentioned methods have not considered the detection of clusterliers, as $k$NN is a sample-wise neighbor concept that is unsuitable for evaluating the abnormal level of clusterliers.

In summary, most existing unsupervised outlier detection methods are either unable to handle scatterliers \cite{OCSVM} or ignore the existence of clusterliers \cite{sheikholeslami1998wavecluster}. A more challenging issue is that the clusterliers can severely hamper the scatterlier detection by interfering the reference set formation of the nearby scatterliers, offering them an excessive number of neighbors and masking them from the $k$NN-based local outlier detection. Therefore, considering the intrinsic correlation between scatterliers and clusterliers and collectively solving the detection of both is the key to achieving robust and accurate outlier detection in complex datasets.

This paper proposes a novel unsupervised outlier detection method to simultaneously explore scatterliers and clusterliers. It introduces a dual reference set formation strategy to form Natural Neighbor Reference Subsets (NRS) and NRS-based Graph Reference Sets (GRS) for comprehensive anomaly index measurement \cite{fenggraph}. The NRS is formed based on the natural neighbors of each sample, and only highly similar samples can appear in the same NRS. By ensuring that the Local Anomaly Index (LAI) is only measured by considering the intra-NRS samples, the masking effect brought by the clusterliers can be considerably mitigated. The GRS is a graph that connects closely distributed NRSs. It acts to reflect the anomaly index of each NRS from a macro distribution perspective. That is, compared to a prominent cluster represented by a larger number of NRSs, a clusterlier is usually a very small and isolated cluster represented by a smaller number of NRSs. Since the NRSs of a clusterlier are not closely connected with the other NRSs by the corresponding GRS, their low connectivity reflected by the GRS makes them easier to detect. As a result, the dual anomaly index more informatively reflects the abnormal degree of each sample to achieve a more robust outlier detection. The four main contributions are summarized as follows:
\begin{itemize}\label{item_Cntrb}
	\item \relax This paper presents a novel unsupervised outlier detection paradigm. To the best of our knowledge, this is the first attempt to simultaneously tackle the detection of both scatterliers and clusterliers.
	
	\item \relax Hierarchical dual reference sets are developed to measure the anomaly indices of samples informatively. As a result, the masking effect caused by the clusterliers is considerably mitigated, effectively enhancing the overall outlier detection accuracy.

        \item \relax The effectiveness of the proposed method has been validated on the clustering tasks. As both scatterliers and clusterliers can be effectively detected and eliminated, better clustering performance can thus be achieved. 
	
	\item \relax Compared to most existing outlier detection methods that are sensitive to the hyper-parameter selection and outlier types, the proposed method demonstrates high robustness, outperforming the counterparts on 32 benchmark datasets.
\end{itemize}

\section{Related Work}\label{sec_RelWr}
\subsection{Density-Based Outlier Detection}
The density information of samples~\cite{zou2025sdenk} has long been widely utilized in outlier detection. Early methods compared the density of each sample in the dataset globally for detection. However, these methods are not suitable for datasets composed of regions with varying densities. Hence, the notion of local density~\cite{cheung2018fast} emerged to address this constraint. Detectors based on the local density assume that outliers should have a lower density compared to their neighbors. Local Outlier Factor (LOF)~\cite{lof} stands as the pioneering local outlier detection method, quantifying the local density of a sample based on its $k$ nearest neighbors. The LOF score of a sample is determined by the ratio of its local density to that of its neighbors. A higher score indicates a greater likelihood of being an outlier. Choosing the parameter $k$ in LOF poses a significant challenge, prompting the development of LOF variants. To tackle the $k$ selection issue, recent works \cite{KFC} have been proposed to evaluate the optimal $k$. However, when clusterliers also exhibit markedly high local density, they may remain undetected by the above methods due to mutual masking.

% clustering-based%
\subsection{Clustering-Based Outlier Detection}
Clustering techniques \cite{zhang2025learningHARR} have been widely employed to detect outliers. Early methods assumed that unclustered samples are outliers. For instance, the FindOut algorithm \cite{yu2002findout} extends the WaveCluster algorithm \cite{sheikholeslami1998wavecluster}, declaring samples outside the detected clusters as outliers. However, this assumption makes it seem more like conducting clustering tasks in noisy environments rather than outlier detection tasks. Additionally, some clustering-based methods first use clustering algorithms to obtain clusters and then employ traditional outlier detection algorithms in clusters to calculate anomaly scores. For example, the work \cite{muhima2020lof} combined $K$-means clustering \cite{kmeans} and LOF algorithms to identify outliers. These methods still target scatterliers. Researchers have gradually observed that outliers can form clusters in some scenarios. CBLOF \cite{CBLOF} considers micro-clusters, independent of large clusters, as a type of outlier. After identifying cluster structures, CBLOF categorizes clusters into large and small clusters. It then measures the anomaly score for each sample based on its cluster size and the distance to the nearest large cluster. In general, the performance of clustering-based methods relies heavily on the chosen clustering algorithm, thus struggling to adapt to datasets with heterogeneous clusters, e.g., clusters with varying sizes, densities, and shapes.

\begin{table}[t]

\centering
\caption{Explanation of Symbols.}
\label{tab:symbols}
\resizebox{\linewidth}{!}{
\begin{tabular}{@{}>{ }l>{ }l>{ }l@{}}
\toprule
\textbf{Section} & \textbf{ Symbol} & \textbf{ Definition} \\
\midrule
\multirow{5}{*}{{ III}} & { $x_i$} & { Data sample} \\
& { $NB(x_i)$} & { Natural neighbor set of $x_i$} \\
& { $kNN_\lambda(x_i)$} & { $k$-nearest neighbors with scope $\lambda$} \\
& { $NBG$} & { Natural Neighborhood Graph} \\
& { $NBG_{ij}$} & { Edge indicator in $NBG$ between $x_i,x_j$} \\
\hline
\multirow{5}{*}{{ IV-A}} & { $\rho(x_i)$} & { Local natural density of $x_i$} \\
& { $dist(x_i, x_j)$} & { Euclidean distance between $x_i,x_j$} \\
& { $S$} & { Set of natural neighbor subsets} \\
& { $s_m$} & { Subset in $S$} \\
& { $U$} & { Upper limit of subset size} \\
\hline
\multirow{7}{*}{{ IV-B}} & { $LAI(x_i)$} & { Local Anomaly Index of $x_i$} \\
& { $\rho_{max}$} & { Max local density in $s_m$} \\
& { $GRS$} & { Graph Reference Sets} \\
& { $LS(s_m, s_w)$} & { Link Strength between $s_m,s_w$} \\
& { $NBP(s_m, s_w)$} & { Natural neighbor pairs between $s_m,s_w$} \\
& { $c_m$} & { Center of $s_m$} \\
& { $SAI(s_m)$} & { Subset Anomaly Index of $s_m$} \\
\hline
\multirow{5}{*}{{ IV-C}} & { $DAI(x_i)$} & { Dual Reference Sets-based Anomaly Index} \\
& { $\beta(s_m)$} & { Weight of $LAI$ for $s_m$} \\
& { $\eta$} & { Sampling rate} \\
& { $T$} & { Sampling times} \\
& { $\mathcal{X}_t$} & { $t$-th sampled dataset} \\

\bottomrule
\end{tabular}
}
\end{table}

\section{Preliminaries}\label{sct:pre}
The proposed method determines reference sets based on the natural neighbor relationship, which is an advanced $k$NN relationship inspired by human friendship. That is, two samples constitute a natural neighbor relationship only when they are mutually within each other's neighborhood set. For clarity, Table~\ref{tab:symbols} lists the key notations and symbols used in this paper. 

Furthermore, several recent methods, including GB \cite{GB}, HDIOD \cite{HDIOD}, and the structurally enhanced GNAN \cite{gnan}, are thoroughly evaluated. The comprehensive results of this evaluation are presented in Section \ref{sec_exp}. Based on the comparative analysis, the NRS framework is selected for integration into DROD. This decision is motivated by its parameter-free natural neighbor concept, which inherently adapts to variable neighbor counts per sample and exhibits strong intra-cluster cohesion \cite{Zhu2016}. These characteristics enhance the method's suitability for robust subset partitioning and reference set construction, while also facilitating automated parameter selection and supporting the formation of hierarchical reference structures.

\begin{definition}\label{def:nb}
[Natural Neighbor - NB] Sample $x_j$ is a natural neighbor of $x_i$ if they satisfy the mutual neighboring condition:
\begin{equation}
    \label{equ_nan}
    x_j \in \mathrm{NB}(x_i) \Leftrightarrow (x_i \in \mathrm{kNN}_\lambda(x_j)) \land (x_j \in \mathrm{kNN}_\lambda(x_i)),
\end{equation}
where $\lambda$ is an eigenvalue (i.e., natural neighbor eigenvalue in \cite{Zhu2016}), signifying the search scope of $k$NN. The value of $\lambda$ is adaptively determined by the natural neighbor search algorithm.
\end{definition}

\begin{definition}\label{def:nbg}
[Natural Neighborhood Graph - NBG] The Natural Neighborhood Graph is a way of representing
the Natural Neighbor relationship of the dataset. Each vertex $v_i$ in the graph represents a data point $x_i$. $x_i$
and $x_j$ are connected if $x_i$ is a natural neighbor of $x_j$. An edge in $NBG$ is defined as:
\begin{equation}
NBG_{ij}=
\begin{cases}1&\mathrm{if~x_i \in NB(x_j) ~or~x_j \in NB(x_i)}\\0&\mathrm{otherwise}.
\end{cases}
\end{equation}
\end{definition}

In this work, NBG stores natural neighbor relationships. After NBG formation, samples without graph-connected natural neighbors can be considered outlier candidates. 

\section{Proposed Method}\label{sct:proposed}
For datasets with outliers belonging to scatterliers and clusterliers, the anomaly index for outlier detection should comprehensively reflect the sample-wise and cluster-wise anomaly levels. The entire dataset is partitioned into subsets based on natural neighbor relationships among data samples to obtain the local reference set for scatterliers. Subsequently, graphs are constructed on these subsets to form macro reference sets for clusterliers. Anomaly indices in the dual hierarchical reference sets are ultimately combined to form a comprehensive quantification of outliers for accurate outlier detection.

The pipeline for scoring abnormality in the proposed method is illustrated in Fig.~\ref{figure-framework}, and is further outlined in the following steps: 1) partition the whole dataset into Natural neighbor subsets based on natural neighbor relationships (Section~\ref{sec_subset_p}), 2) construct Graph Reference Sets (GRS) to measure the anomaly indices of subsets by analyzing their distribution (Section~\ref{sec_as}), 3) treat Natural Neighbor Subsets as Natural Neighbor Reference Subsets (NRS) and calculate the LAIs of samples within each NRS (Section~\ref{sec_as}), and 4) form the overall anomaly indices by combining the two types of indices obtained in 2) and 3) (Section~\ref{sct:detection}).

\begin{figure*}[t]
	\centering

         \includegraphics[width=0.9\linewidth]{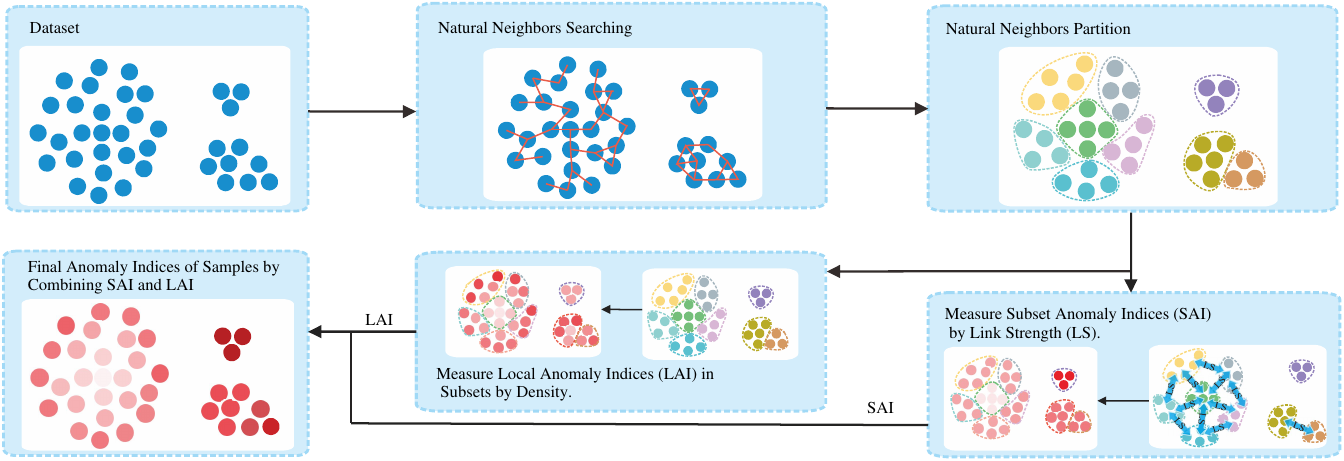}
   
\caption{Pipeline of the proposed method. Natural neighbor subsets containing natural neighboring samples are constructed based on the NB definition in Section~\ref{sct:pre}. Then, the Subset Anomaly Index (SAI) and Local Anomaly Index (LAI) are computed to comprehensively reflect the sample abnormality under the complex co-occurrence of scatterliers and clusterliers. SAI is derived from the degree of adjacency (i.e., link strength) among subsets, indicating the subset-level abnormality of samples. LAI measures the local abnormality of samples within each subset based on distribution density. These two indices collectively enable the detection of outliers without bias and clusterliers' masking effect.}
	\label{figure-framework}
\end{figure*}
\subsection{Natural Neighbor Subset Exploration}\label{sec_subset_p}

In this subsection, a dataset with $n$ data samples $X=\{x_1,\ldots,x_i,\ldots,x_n \}$ is divided into a series of $m$ natural neighbor (NB) subsets $S=\{s_1,\ldots,s_m\}$, which will be treated as micro reference sets and help construct macro graph reference sets in the following. The NB subsets exploration starts from dense samples according to the NB definition. Note that the final number of subsets $m$ is automatically determined by the NB relationships and data distribution.

Before the subset partition, it is necessary to measure the local density of samples for obtaining the subset exploration priority. Traditionally, in $k$NN, the local density of a sample can be simply considered as the inverse of the average distance to its $k$ neighbors. However, for NBs, the number of neighbors for each sample varies, unlike a relatively fixed $k$ in $k$NN. Therefore, measuring local density should consider the quantity of NBs rather than relying solely on average distance. Accordingly, based on the NB definition in Definition~\ref{def:nb}, the local density of a sample is defined as follows.

\begin{definition}\label{def:den}
[Local Density for NBs - $\rho$] 
In the context of NB analysis, the local density of $x_i$ is defined as
\begin{equation}
    \label{equ_rou}
    \rho(x_i)= \frac{|\text{NB}(x_i)|}{\sum\limits_{x_j\in \mathrm{NB}(x_i)}dist(x_i,x_j)},
\end{equation}
where $|\mathrm{NB}(x_i)|$ is the number of samples in $\mathrm{NB}(x_i)$, and 
\begin{equation}
    dist(x_i, x_j)=||x_i-x_j||_2
\end{equation}
calculates the Euclidean distance between two sample vectors $x_i$ and $x_j$. 
\end{definition}

\begin{algorithm}[t]
% \fontsize{12}{5}\selectfont 
    \caption{Efficient natural neighbor search}
    \label{algo_NaN}
    \KwIn{Dataset $X=\{x_1, x_2,...,x_n\}$}
    \KwOut{Natural Neighborhood Graph ($\mathrm{NBG}$)}
    Initialize: $r = 1$, $\mathrm{NBG} = 0 \in \mathbb{R}^{n \times n}$ \;
    $\forall x_i \in X$, set the corresponding number of NBs by $\mathrm{NuB}(x_i) = 0$ \;
    Create a k-d tree $T$ from $X$ \;
    \While{True}{
        \For{each $x_i \in X$}{
            Fetch the $r$ nearest neighbors $\mathrm{kNN}_r(x_i)$ of $x_i$ according to $T$ \;
        }
        \For{each $x_i \in X$}{
            \For{each $x_j \in \mathrm{kNN}_r(x_i)$}{
                \If{$x_i \in \mathrm{kNN}_r(x_j)$}{
                    $\mathrm{NBG}_{ij} = 1$, $\mathrm{NBG}_{ji} = 1$ \;
                }
            }  % End of inner loop for $x_j$
            $\mathrm{NuB}(x_i) =$ number of NBs of $x_i$ \;
        }
        Count the number of samples without a natural neighbor, denoted as $Count$ \;
        \If{$Count == 0$ or remain unchanged}{
            \textbf{break}
        }
        Increase the hop range by $r = r + 1$ \;
    }
\end{algorithm}

A sample with more closely distributed NBs exhibits higher local density. The NB search follows the descending order of sample density, as low-density samples often lack natural neighbors. Placing them late in the queue enables early termination to save computation. The search uses a k-d tree \(T\) of \(X\). By gradually increasing \(r\), more neighbors are considered as NB candidates. Once NBs stabilize, the search ends, and the NB relationships form the Natural Neighborhood Graph (NBG) in Section~\ref{sct:pre}. The entire process is outlined in Algorithm~\ref{algo_NaN}.

With the searched NBs represented by the NBG, the whole dataset is partitioned into natural neighbor subsets containing naturally similar samples that will be utilized to form reference sets for accurate outlier detection. Since the isolated samples may have fewer or even no NBs, they form Natural neighbor subsets by themselves. Similar to the NB search, the exploration of Natural neighbor subsets also starts from the samples with higher density, and the samples with very low density are treated as singleton subsets. 

\begin{algorithm}[t]
	\caption{Natural neighbor subset exploration}
	\label{algo_subsetP}
        \KwIn{Dataset $X=\{x_1, x_2,...,x_n \}$ }
        \KwOut{A series of natural neighbor subsets $S$}
        %Initialize: $S=\emptyset$, upper limitation $U=\sqrt{n}$ \;
        Initialize: $S=\emptyset$, upper limit $U$ \;
        Obtain $\mathrm{NBG}$ by the efficient NB search algorithm in Algorithm~\ref{algo_NaN} \;
        Obtain density $\rho(x_i)$  of each sample $x_i$  by Eq.~\eqref{equ_rou}\;
        Initialize unassigned set $R=\{x_i\mid x_i \in X  \cap \rho(x_i) > 0 \}$\;
        \While{$R \neq \emptyset $}{
            $x_m = \arg\max_{x_i \in R}  \rho(x_i) $ \;
            $R = R \setminus x_m $ \;
            Initialize a new subset $s_{new} = \emptyset$ \;
            $s_{new} = s_{new} \cup x_m $\;
            \For{each $x_j$ in $s_{new}$}{
            Obtain NB($x_j$) by $\mathrm{NBG}$\;
            Add unassigned neighbors in NB($x_j$) to $s_{new}$ and remove them from $R$\;
            \lIf{$size(s_{new}) \geq U$}{\textbf{break}}
            }
            $S = S \cup s_{new}$ \;
        }
\end{algorithm}

More specifically, for the samples outside the currently formed subsets, the one with the highest density $\rho$ is selected, and a new natural neighbor subset $s_{\text{new}}$ is created for it. To enforce strict similarity within subsets, the natural neighbors of all samples in $s_{\text{new}}$ are aggregated into a candidate set and then merged into $s_{\text{new}}$. The current subset \( s_{\text{new}} \) is finalized and closed once its size reaches a predetermined upper limit \( U \). This parameter acts as a termination threshold to avoid oversized subsets that would otherwise compromise local similarity. A typical setting for \( U \) is \( \sqrt{n} \) \cite{ester1996density}, a choice supported by the natural neighbor property that ensures high intra-subset similarity. This value also aligns with the empirical rule widely adopted in density-based clustering to balance subset granularity and computational efficiency. Specifically, closure occurs if $|s_{\text{new}}| \geq U$ or if no unassigned natural neighbors remain (i.e., $R = \emptyset$). The process then continues recursively using the same procedure to form subsequent subsets until all samples are assigned. The complete natural neighbor subset exploration process is outlined in Algorithm~\ref{algo_subsetP}.

\subsection{Dual Reference Sets-based Anomaly Index} \label{sec_as}

All the NBs in $S$ obtained through Algorithm~\ref{algo_subsetP} are treated as micro reference sets named Natural Neighbor Reference Subsets (NRS) to reflect the local anomaly degree of samples without being masked by the surrounding clusterliers. Specifically, in each NRS, the abnormality of each sample is measured in terms of density, where lower density indicates a higher anomaly. The Local Anomaly Index (LAI) of a sample is determined by the difference between its density and the density peak in the subset to which it belongs, which can be written as
\begin{equation}
    \label{equ_LAS}
    \mathrm{LAI}(x_i)=\rho_{max}-\rho(x_i),
\end{equation}
where $\rho_{max}$ is the max local density in $s_m$ that $x_i$ belongs to. In the micro cluster composed of clusterlier samples, the center point has a relatively high density. If the scatterliers are in the same NRS, they can be effectively identified due to their high LAI value. 

Clusterliers, typically characterized as small and isolated clusters with a limited number of NRSs, are detected using NRS-based Graph Reference Sets (GRS). These graphs connect compactly distributed NRSs and serve as macro-level reference sets to reflect the anomaly degree of each entire subset (i.e., each NRS). To form the GRSs, Link Strength (LS) is defined to quantify the connectivity compactness between NRSs.
\begin{definition}\label{def:ls}
[Link Strength - LS] For two NRSs $s_m$ and $s_w$, the degree of their adjacency considers both the distance between the subsets and the number of NB relationships between their respective members. The LS between them is computed by
\end{definition}
\begin{equation}
    \label{equ_LS}
    \mathrm{LS}(s_m,s_w)=\frac{\mathrm{NBP}(s_m,s_w)}{dist(c_m,c_w)},
\end{equation}
where $\mathrm{NBP}(s_m, s_w)$ counts the number of sample pairs with the NB relationship between $s_m$ and $s_w$, $dist(c_m,c_w)$ computes the Euclidean distance between the two NRS centers $c_m$ and $c_w$, which are the sample means of $s_m$ and $s_w$ computed by 
\begin{equation}
    c_m=\frac{1}{|s_m|}\sum_{x_i\in s_m}x_i\ \ \text{and}\ \ c_w=\frac{1}{|s_w|}\sum_{x_j\in s_w}x_j,
\end{equation}
respectively. With LS, each NRS can be viewed as connected with NRSs that have an NB relationship to it, i.e., $\mathrm{NBP}(s_m,s_w)\neq 0$, and each graph formed in this way is denoted as a GRS.

With the constructed GRSs, the anomaly index of an NRS can be determined by the overall LS measured between it and all the other connected NRSs in the GRS. The Subset Anomaly Index (SAI) is computed and normalized into the range $[0, 1]$ by
\begin{equation}
    \label{equ_SAS}
    \mathrm{SAI}(s_{m})=1-norm\left(\sum_{s_w\in S,s_w\neq s_m}\mathrm{LS}(s_m,s_w)\right),
\end{equation}
where $norm(\cdot)$ indicates the min-max normalization function. Note that a higher SAI value indicates a lower LS connectivity, which indicates that the NRS is relatively isolated from the perspective of the other NRSs. Therefore, a higher SAI value reflects that all the samples in the NRS have higher anomaly degrees. Building upon these micro-level (LAI) and macro-level (SAI) anomaly indicators within NRSs and GRSs respectively, a comprehensive anomaly measurement is achieved through their integration into the Dual Reference Sets-based Anomaly Index. The detailed formulation of this combined metric is presented in Section~\ref{sct:detection}.

\subsection{Outlier Detection with Sampling Enhancement}\label{sct:detection}

The overall Dual Reference Sets-based Anomaly Index (DAI) of each sample is computed based on the LAI and SAI by
\begin{equation}
    \label{equ_AS}
    \mathrm{DAI}(x_i)=\mathrm{SAI}(s_m)+\beta(s_{m}) \cdot \mathrm{LAI}(x_i),
\end{equation}
where $s_m$ is the NRS to which $x_i$ belongs and $\beta(s_{m})$ controls the contribution of LAI. It should be noted that the two indices hierarchically reflect anomaly degrees, where the macro SAI is the basis reflecting the degree of the whole subset, while the micro LAI is the finer indicator revealing the position of each single sample within the subset. Accordingly, the weight of LAI is set at $\beta(s_{m}) = \mathrm{SAI}(s_{m})$.

The spatial patterns associated with different combinations of high and low SAI and LAI values are systematically analyzed to determine their indicative sample types. Four possible value configurations and their preliminary interpretations are summarized in Table \ref{tab:sai_lai_combinations}. 
\begin{table}[t]
\centering
\caption{Sample types under different SAI and LAI combinations.}
\scalebox{0.95}{
    \begin{tabular}{c|cc|c|l} 
    \toprule
    \textbf{Case} & \textbf{SAI} & \textbf{LAI} & \textbf{Anomaly Type} & \textbf{Description} \\
    \midrule
    1 & High & High & Scatterlier & \makecell[l]{Clear abnormality with dual \\evidence} \\
    \midrule 
    2 & High & Low & Clusterlier & \makecell[l]{Core member of anomalous\\ micro-cluster} \\
    \midrule
    3 & Low & High & Potential noise & \makecell[l]{Local isolation, global\\ normality} \\
    \midrule
    4 & Low & Low & Normal sample & \makecell[l]{No significant anomaly\\ evidence} \\
    \bottomrule
    \end{tabular}
}
\label{tab:sai_lai_combinations}
\end{table}
  These relationships are further elaborated in the following remarks:
\begin{remark}
  The interaction between LAI and SAI corresponds to four distinct sample types:
\end{remark}
\begin{enumerate}

\item   \textit{High LAI, High SAI:} Indicates both local and global isolation, confirming a \textit{scatterlier} with high confidence.

\item   \textit{High SAI, Low LAI:} Suggests global subset isolation with local cohesion, characteristic of a \textit{clusterlier} core member.

\item   \textit{High LAI, Low SAI:} Implies local isolation within a globally connected subset, often representing noise within normal clusters.

 \item   \textit{Low LAI, Low SAI:} Reflects full integration at both levels, indicating a normal sample.

\end{enumerate}

  These types underpin the development of DAI, which integrates LAI and SAI into a single anomaly score. To verify the practical value of DAI, the Dual Reference set-based Outlier Detection (DROD) algorithm is designed to compute DAI, and its effectiveness is experimentally validated through the visual analysis of DAI, LAI, and SAI distributions in Section~\ref{subsection_daivalidity}.

To enhance the robustness of DAI, a sampling enhancement mechanism is introduced, which performs \( T \) random samplings on the dataset \( X \) at a rate \( \eta \). The DAI is computed within each sampled subset and aggregated to produce the final anomaly score. The effectiveness of sampling is demonstrated in the following cases: 
\begin{enumerate}
    \item Normal samples and clusterliers, which are associated with relatively high density in their Natural Reference Sets (NRSs), remain largely unaffected by sampling;
    \item Scatterliers, which are sparse and form weakly connected NRSs and GRSs, become more isolated after sampling, leading to an increase in DAI scores. The overall sampling-enhanced DROD method is outlined in Algorithm~\ref{algo_OD1}.
\end{enumerate}

\begin{algorithm}[t]
	\caption{Dual reference set-based outlier detection}
	\label{algo_OD1}
        \KwIn{Original dataset $X=\{x_1, x_2,...,x_N \}$,  sampling rate $\eta$, sampling times $T$ }
        \KwOut{Anomaly index $\mathrm{DAI}_{X}$ of all the samples in $X$}
        $\mathrm{DAI}_{X_t} = 0 \in R^N$;\\
        Initialize $t$=0, upper limit T\\
        \While{$t \leq T$}{
        Obtain the \( t \)-th sampled dataset ${X}_{t}$ from $X$ with sampling rate $\eta$\;
        Obtain natural neighbor subsets $S$ and $\mathrm{NBG}$ on ${X}_{t}$ by Algorithm~\ref{algo_subsetP}\;
        Calculate $\mathrm{SAI}$ of each subset in $S$ by Eq.~\eqref{equ_SAS}\;
        Compute $\mathrm{LAI}$ of each sample in ${X}_{t}$ by Eq.~\eqref{equ_LAS}\;
        Calculate $\mathrm{DAI}_{X_t}$ of each sample in ${X}_{t}$ by Eq.~\eqref{equ_AS}\;
        \For{each $x_i$ in $X$}{ 
        $\mathrm{DAI}_{X_t}[x_i]$ = $\mathrm{DAI}_{X_t}[x_i]$ + $\mathrm{DAI}_{X}[x_i]$\;
            }
        $t = t + 1$\;
        }
\end{algorithm}

\begin{theorem}
The time complexity of DROD is $O(T \cdot N \cdot d \cdot \log N)$, where $T$ represents the number of sampling iterations, $N$ is the size of the original dataset, and $d$ is the dimensionality of the data.
\end{theorem}

\begin{proof}
DROD repeatedly and systematically evaluates multiple random samplings. For each sampled dataset $X_t$ (with size $n = \eta \cdot N$, where $0 < \eta \leq 1$), the evaluation includes mainly obtaining natural neighbor subsets $S$ via the proposed method, and calculating the three anomaly indices: SAI, LAI, and DAI.

The time complexity for obtaining $S$ using a $k$-d tree is approximately $O(n \cdot d \cdot \log n)$. This includes neighbor search using the tree structure ($O(n \cdot d \cdot \log n)$), local density measurement ($O(n)$), and natural neighbor subset construction ($O(n)$). For index calculation: SAI requires $O(n \cdot d)$ operations for distance computations to subset centroids, while both LAI and DAI require $O(n)$ each. Thus, the total time complexity per sampled dataset is approximately $O(n \cdot d \cdot \log n)$. 
\end{proof}

\section{Experiments}\label{sec_exp}
To illustrate the effectiveness of the proposed DROD unsupervised outlier detection method, eight experiments are conducted: 1) outlier detection performance evaluation (Section~\ref{subsection_OD_4}) on both real benchmark and synthetic datasets with significance tests, 2) computation efficiency evaluation (Section~\ref{subsection-efficiency}), 3) evaluation of NRS and alternative reference sets (Section~\ref{subsection_nrs}), 4) validation of potential distance metric (Section~\ref{subsection_distance}), 5) ablation studies (Section~\ref{subsection_ablation}), 6) enhancement on downstream clustering task (Section~\ref{subsection_downstream}), 7) hyper-parameter evaluations (Section~\ref{subsection_parameter}), and 8) study of the relationships among DAI, SAI, and LAI (Section~\ref{subsection_daivalidity}). The experimental settings are first introduced, followed by a comprehensive presentation of results along with in-depth analysis.

\subsection{Experimental Settings}\label{subsection_OD_1}
This section introduces the dataset and the corresponding experimental design, evaluation metrics, and compared methods. All experiments were conducted on a workstation with 32 GB RAM and a 2.3 GHz 11th Gen Intel(R) Core(TM) i7-11800H CPU using Python 3.11.

\subsubsection{Datasets}

\begin{table}[t]
\centering
	\caption{Statistics of 20 real benchmark datasets.}
	\scalebox{0.85}{
	\begin{tabular}{l|rrrrr} % centered columns (4 columns)
	\toprule
	 \textbf{Real dataset}                 & \textbf{Samples} & \textbf{ Features} & \textbf{ Anomaly(\%)} & \textbf{Reference} \\
	\midrule
PageBlocks         & 5393    & 10       & 510 (9.46)       & \cite{malerba1996further}    \\
WPBC             & 198     & 33       & 47 (23.74)        & \cite{mangasarian1995breast}      \\
mnist            & 7603    & 100      & 700 (9.21)       & \cite{lecun1998gradient}   \\
musk           & 3062    & 166      & 97 (3.17)        & \cite{dietterich1993comparison}            \\
Ionosphere        & 351     & 33       & 126 (35.90)       &  \cite{sigillito1989classification}                \\
Waveform           & 3443    & 21       & 100 (2.90)       & \cite{loh2011classification}  \\
% Cardiotocography 简写为Cardiotoco.
Cardiotoco.   & 2114    & 21       & 466 (22.04)       &\cite{ayres2000sisporto}                   \\
cardio           & 1831    & 21       & 176 (9.61)       & \cite{ayres2000sisporto}               \\
landsat                         & 6435    & 36       & 1333 (20.71)      &\cite{meta_analysis}                  \\
optdigits                            & 5216    & 64       & 150 (2.88)       &\cite{alpaydin1998cascading}               \\
pendigits                            & 6870    & 16       & 156 (2.27)       &\cite{alimoglu1997combining}                \\
speech                               & 3686    & 400      & 61 (1.65)        & \cite{brummer2012description}               \\
thyroid                              & 3772    & 6        & 93 (2.47)        &\cite{quinlan1987inductive}              \\
Pima          & 768     & 8        & 268 (34.90)       & \cite{Rayana2016}                       \\
satellite        & 6435    & 36       & 2036 (31.64)      & \cite{Rayana2016}           \\
satimage-2                           & 5803    & 36       & 71 (1.22)        & \cite{Rayana2016}            \\
vowels                               & 1456    & 12       & 50 (3.43)        &     \cite{kudo1999multidimensional}     \\
Seismic                               & 2584    & 11       & 170 (6.58)        &  \cite{rod2022}       \\
Banknote                               & 1372    & 4       & 610 (44.46)        &  \cite{rod2022}      \\
%HeartDisease 简写为HeartDis.
HeartDisease                               & 270    & 13       & 120 (44.44)        &  \cite{campos2016evaluation}      \\
    \bottomrule
	\end{tabular}
	}
	\label{table:datasets} % is used to refer to this table in the text
\end{table}

\begin{figure}[t]
	\centering
	\includegraphics[width=0.9\linewidth]{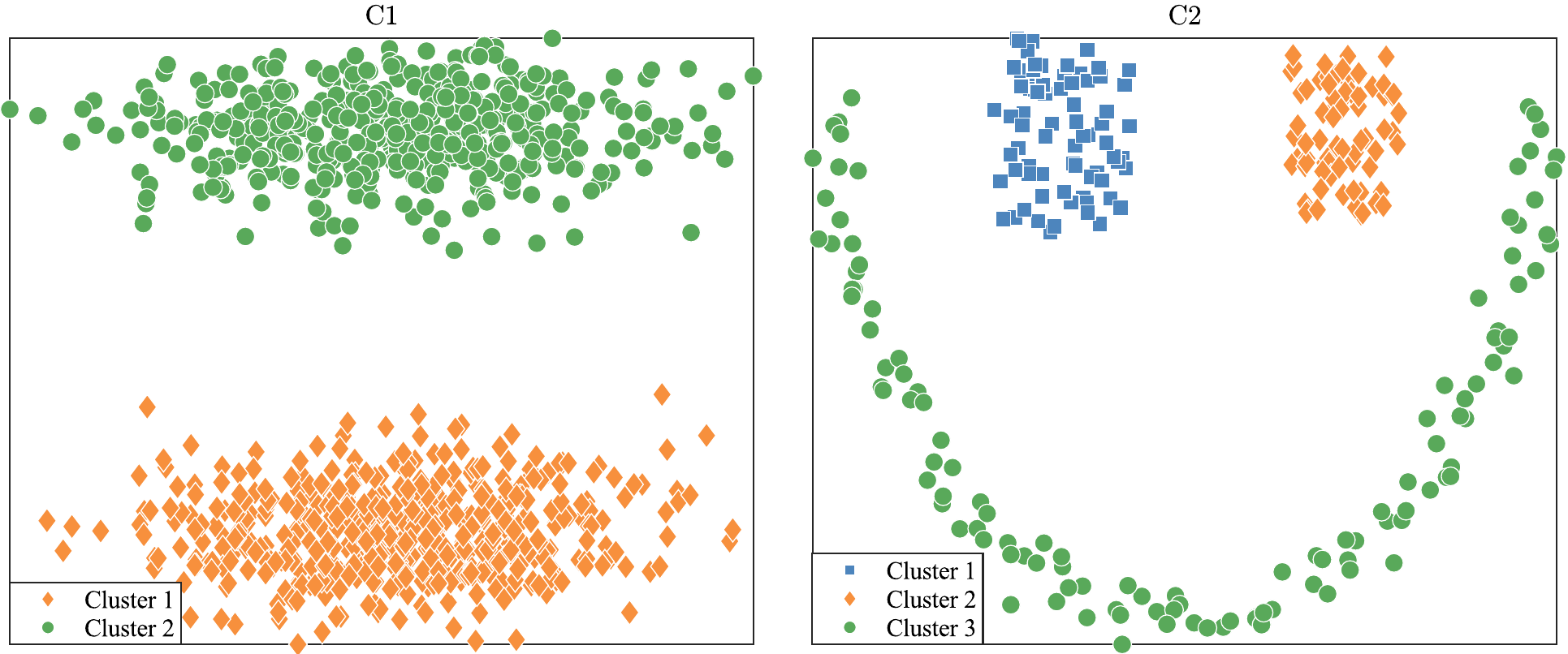}
	\caption{Two basic synthetic datasets C1 and C2 for generating various types of outliers.}
	\label{figure-clusters}
\end{figure}

To comprehensively evaluate the performance of our method, 20 real benchmark datasets and 12 synthetic benchmark datasets with generated outliers are utilized in the experiments. The statistical information and source of the real datasets are summarized in Table~\ref{table:datasets}. All the 12 synthetic datasets are generated based on two commonly used outlier detection datasets denoted as C1 and C2, composed of clusters in different shapes\footnote{https://github.com/milaan9/Clustering-Datasets}. The two base datasets are visualized in Fig.~\ref{figure-clusters}. The samples in C1 and C2 are treated as normal, and new samples are generated as outliers encompassing two types: scatterliers and clusterliers. The generation mechanisms are as follows:
\begin{itemize}\label{item_outliersgen}
	\item \relax Scatterliers: The value for each dimension of an outlier is randomly generated from a uniform distribution, $Unif(x_{mean} - 1.5 * range, x_{mean} + 1.5 * range)$, where $x_{mean}$ denotes the mean of all normal data points, and $range$ represents the maximum distance of any point from the mean, specifically, $range = \max\!\big(|x_{max} - x_{mean}|,\; |x_{mean} - x_{min}| \big)$.
	\item \relax Clusterliers: A Gaussian distribution with designated mean and variance is utilized for clusterliers generation, where the size of the micro clusters containing clusterliers usually cannot exceed 10\% of the size of the whole dataset.
\end{itemize}

\begin{figure*}[t]
  \centering
  \includegraphics[width=0.9\linewidth]{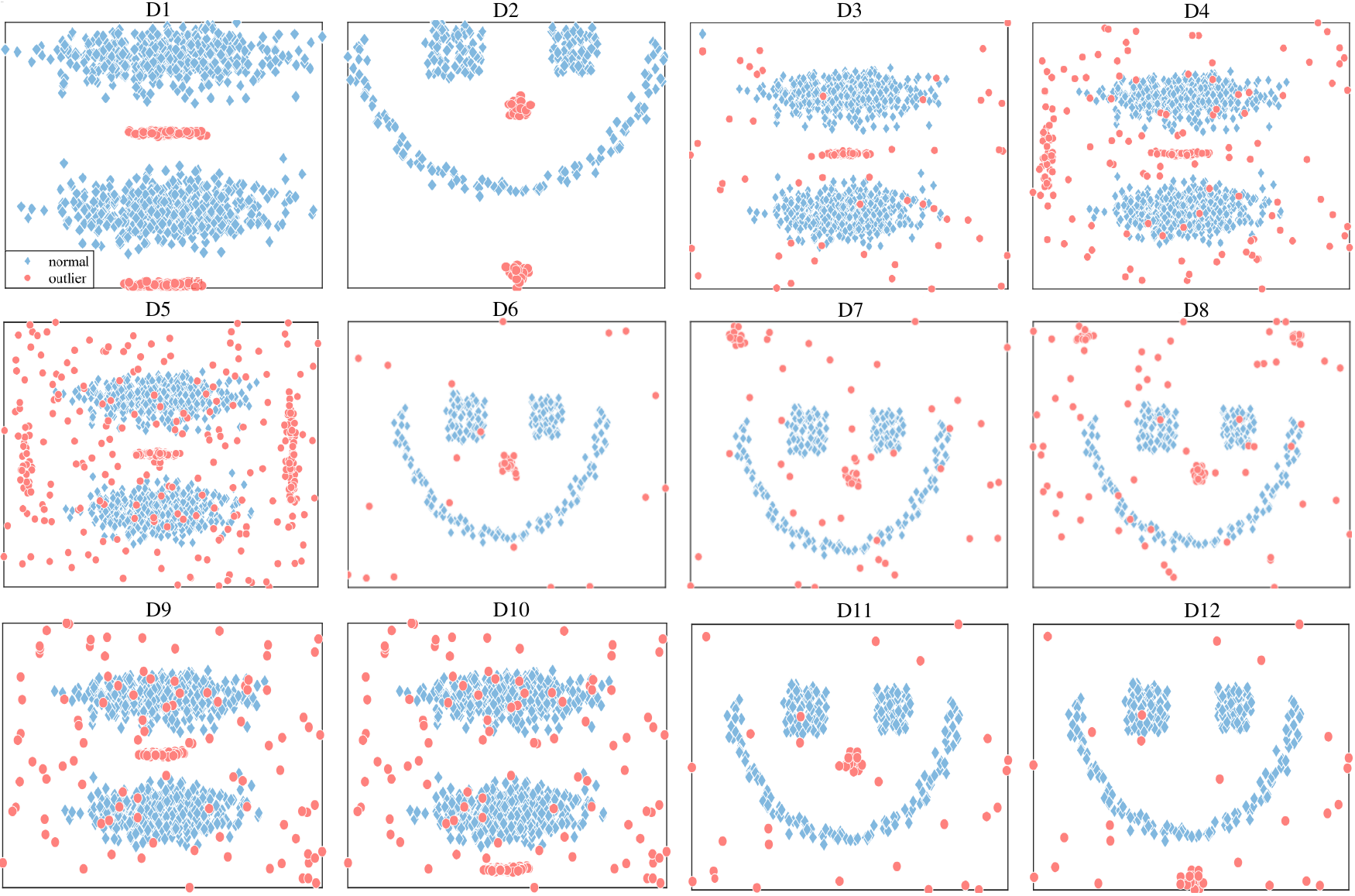}

    \caption{Visualization of synthetic datasets D1-D12.}

  \label{figureD12}
\end{figure*}

All 12 synthetic datasets with outliers are summarized in Table~\ref{table:D1D12}. A visualization of the datasets is provided in Fig.~\ref{figureD12}, where D1 and D2 are generated to evaluate the clusterlier detection performance of different methods, while D3 to D12 contain a mixture of scatterliers and clusterliers.
\begin{table}[t]
\centering
%\footnotesize
	\caption{Statistics of 12 synthetic datasets and the counts of scatterliers and clusterliers.} % title of Table
	%\footnotesize
	\scalebox{1}{
	\begin{tabular}{l |c|cr} % centered columns (4 columns)
	\toprule
	 \textbf{Dataset}   & \textbf{Normal Samples} & \textbf{Scatterliers} & \textbf{Clusterliers} \\
	\midrule
D1         & 1000    & 0       & (100,100)           \\
D2         & 266     & 0       & (26,26)           \\
D3         & 1000    & 50       & (50)           \\
D4         & 1000    & 100       & (50,50)           \\
D5         & 1000    & 150       & (30,50,70)           \\
D6         & 266    & 20       & (20)           \\
D7         & 266    & 40       & (20,20)           \\
D8         & 266    & 60       & (10,20,30)           \\
D9         & 1000    &100        & (50)           \\
D10         & 1000    &100        & (50)           \\
D11         & 266    &26        & (26)           \\
D12         & 266    &26        & (26)           \\
    \bottomrule
	\end{tabular}
	}
	\label{table:D1D12} % is used to refer to this table in the text
\end{table}
\subsubsection{Evaluation Metrics}\label{subsection_OD_2}

Four validity indices have been utilized to evaluate the performance of the compared methods, which are described as follows:

\begin{itemize}
    \item \textbf{Area Under the ROC Curve (AUC): }In the field of outlier detection, the Area Under the Receiver Operating Characteristic (ROC) Curve (AUC) is one of the most popular evaluation metrics. The ROC curve plots the true-positive rate against the false-positive rate over varying thresholds. Its value ranges in $[0,1]$ and reflects the sorting quality of an algorithm. The higher the AUC, the better the performance of the outlier detector.

    \item \textbf{Precision-$s$: }The Precision-$s$ \cite{campos2016evaluation}, a specialized variant of the standard Precision metric, is defined as the proportion of correct results within the top $s$ ranks, serves as the second evaluation metric. This metric requires the setting of the threshold $s$ to match the number of outliers in the ground truth.

    \item \textbf{Davies-Bouldin Index (DBI): }The Davies-Bouldin Index (DBI) \cite{DBI}, which considers both intra-cluster and inter-cluster sample similarity, is used to measure clustering performance, and a smaller value indicates better performance.

    \item \textbf{Wilcoxon Signed-Rank Test: }Additionally, the Wilcoxon signed-rank test \cite{Wilcoxon2013}, a nonparametric statistical method, is used to reveal significant differences between algorithms. The null hypothesis assumes there are no differences between the algorithms. If the hypothesis is rejected in the comparison of DROD and a counterpart where DROD performs better, it indicates that DROD demonstrates significant superiority in comparison with the counterpart.

\end{itemize}

\subsubsection{Compared Methods}\label{subsection_OD_3}

The proposed DROD method is compared with eight representative methods, including conventional baselines and advanced recent approaches. It is worth noting that the focus is on unsupervised outlier detection, where most existing label-guided outlier detection methods are not applicable. Therefore, all compared methods are either specifically designed for unsupervised scenarios or demonstrate effective capability in such settings. Regarding experimental parameters, the proposed method employs a sampling rate $\eta=0.8$ and sampling times $T=60$, which will be validated in the subsequent hyper-parameter sensitivity evaluations.

Specifically, three $k$NN-based methods are compared: $k$NN, LOF \cite{lof}, and DGOF \cite{DGOF}. All of these methods utilize the state-of-the-art KFC \cite{NA} framework to automatically determine the optimal $k$ value. Additionally, five other methods are evaluated: CBLOF \cite{CBLOF,CBLOFNEW}, OCSVM \cite{OCSVM,OCSVMNEW}, IFOREST \cite{D-iforest}, COPOD \cite{COPOD}, and ECOD \cite{Li2022ECOD}. For CBLOF, the parameter is set to ``n\_clusters=10''. For OCSVM, the kernel width $c$ is determined through cross-validation, and the $\nu$ value is selected based on the tolerance for outliers. For IFOREST, the parameter is set to ``estimators=10''. The remaining two advanced methods, ECOD and COPOD, are parameter-free. All parameters of the compared methods, where applicable, are configured according to the recommendations in their original publications to ensure a fair comparison.

\subsection{Outlier Detection Performance Evaluation } \label{subsection_OD_4}

The outlier detection performance of different methods is evaluated on complex real benchmark datasets and synthetic datasets with specifically generated scatterliers and clusterliers.
\begin{table*}[t]
  \centering
  \caption{AUC comparison of 9 methods on D1 and D2 with \textbf{only clusteriers}. The best results are emphasized in \textbf{bold}.}
\resizebox{\linewidth}{!}{
    \begin{tabular}{l|ccccccccc}
    \toprule
      \textbf{Datasets}         & \textbf{COPOD} & \textbf{ECOD} & \textbf{OCSVM} & \textbf{IFOREST} & \textbf{$k$NN+KFC} & \textbf{LOF+KFC} & \textbf{DGOF+KFC} & \textbf{CBLOF} & \textbf{DROD}  \\ 
   &\textbf{\cite{COPOD}'2020} & \textbf{\cite{Li2022ECOD}'2022} &\textbf{\cite{OCSVMNEW}'2023} & \textbf{\cite{D-iforest}'2023} & \textbf{\cite{NA}'2024} & \textbf{\cite{NA}'2024} & \textbf{\cite{NA}'2024} & \textbf{\cite{CBLOFNEW}'2025} & \textbf{ours}\\
    \midrule
        D1 & 0.5070 & 0.4149 & 0.4932 & 0.6759 & 0.1266 & 0.4887 & 0.3562 & 0.2415 & \textbf{0.8330} \\ 
        D2 & 0.5227 & 0.3758 & 0.5000 & 0.5454 & 0.8228 & 0.5856 & 0.6076 & 0.0753 & \textbf{0.9180}  \\ 
        \midrule  
        Avg. AUC & 0.5149 & 0.3954 & 0.4966 & 0.6107 & 0.4747 & 0.5732 & 0.4819 & 0.1573 & \textbf{0.8755} \\   
    \bottomrule
    \end{tabular}%
  }
  \label{table:D1D2}
\end{table*}
\subsubsection{Clusterliers Detection Performance Evaluation}

The AUC results on D1 and D2 are reported in Table~\ref{table:D1D2}. Both datasets contain only cluster-type anomalies, meaning that anomalous samples are densely grouped rather than individually isolated. In this setting, most compared methods produce AUC values close to 0.5, indicating that they are largely unable to distinguish clustered anomalies from normal data. This is consistent with the known limitation of local density–based models such as LOF, ECOD, and COPOD, in which mutually similar anomalies tend to reinforce each other and thus remain undetected. In contrast, DROD consistently achieves the highest AUC on both datasets. This superior performance can be attributed to the macro reference mechanism, where clusterliers are grouped into subsets that exhibit weak connectivity to the global structure, resulting in a high SAI response. Although IFOREST is not specifically designed for clustered anomalies, it attains the second-best result on D1 because its recursive partitioning strategy may accidentally isolate marginal clusters, partially revealing their abnormality. Similarly, $k$NN+KFC performs competitively on D2 due to the automatically optimized neighborhood size, which occasionally captures the boundary structure of compact anomaly clusters. These observations confirm that traditional single-scale detectors are prone to masking effects under densely grouped anomalies, whereas DROD avoids such limitations by explicitly evaluating the global separability of subset structures.

\subsubsection{Scatterliers and Clusterliers Detection Evaluation}

Fig.~\ref{figure-zhu} reports the AUC results of all nine methods on synthetic datasets D1–D12. These datasets progressively incorporate both scatterliers and clusterliers in varying proportions and spatial configurations, making them suitable for evaluating robustness under heterogeneous anomaly patterns. Overall, DROD consistently maintains the highest AUC across most datasets, indicating that its dual reference formulation is capable of adaptively capturing both types of anomalies under mixed distributional settings. A closer inspection reveals that certain conventional methods exhibit sensitivity to the dominant anomaly type. For example, $k$NN and IFOREST perform comparatively well on D4, where anomalies are moderately separated and the local neighbor structure remains partially intact. DGOF achieves competitive results on D6 and D7, where the proportion of scatterliers is small and clusterliers are relatively compact. In contrast, OCSVM and ECOD show strong performance on D10, which contains clearer global separation between anomalous and normal clusters. These method-specific fluctuations suggest that single-scale or single-assumption detectors tend to succeed only when their inductive bias matches the underlying anomaly structure. It is worth noting that all synthetic datasets are intentionally designed to share identical normal cluster structures (as shown in Fig.~\ref{figure-clusters}), with only the anomaly patterns varying. This ensures that DROD is evaluated under controlled anomaly distribution shifts, where its stability is preserved by design. This behavior is aligned with its design: the dynamic combination of LAI and SAI avoids over-committing to either local or global evidence, thus mitigating the brittleness commonly observed in purely local or global detectors.
\begin{figure}[t]
	\centering
	\includegraphics[width=0.9\linewidth]{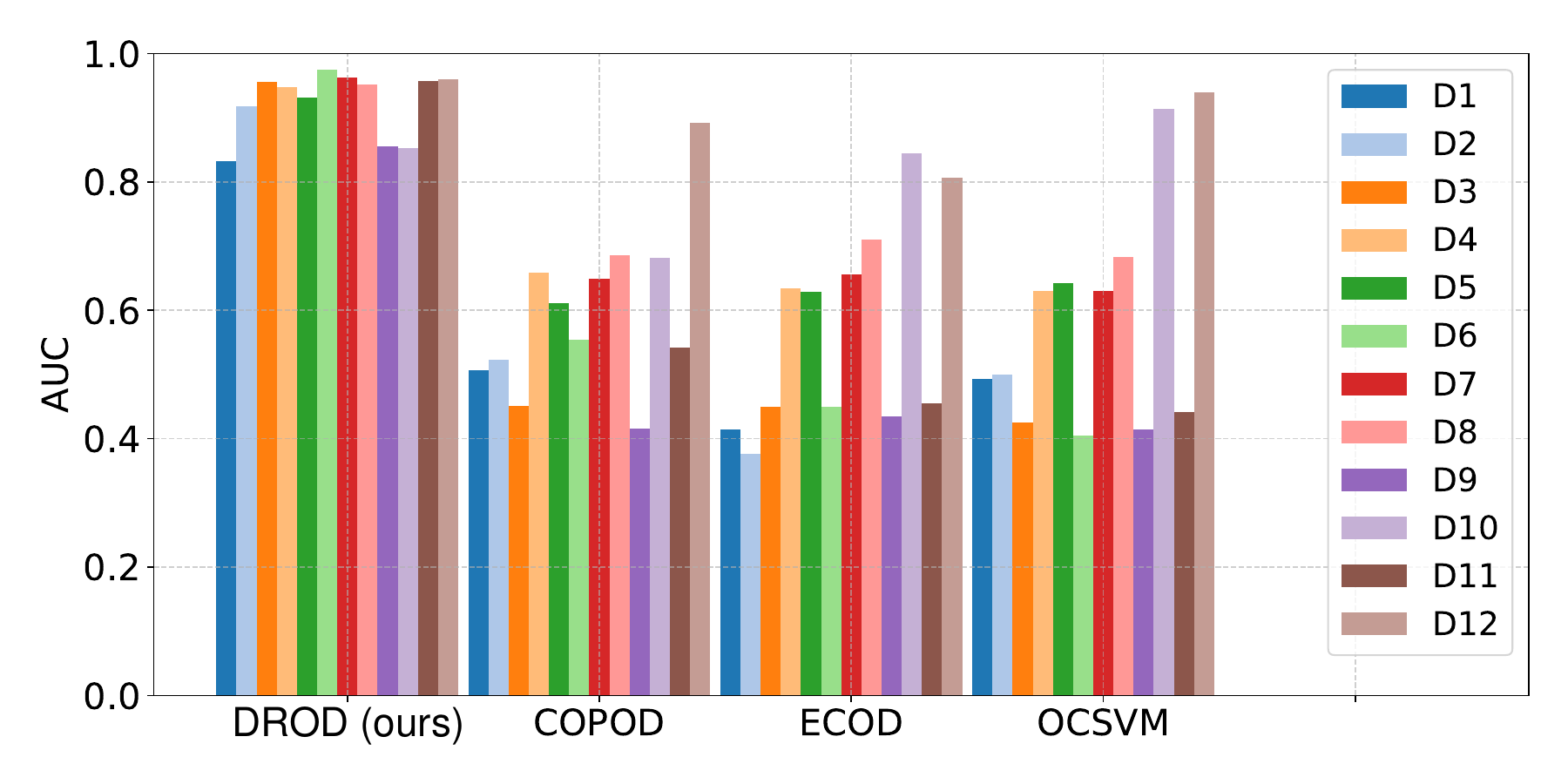}
        \includegraphics[width=0.9\linewidth]{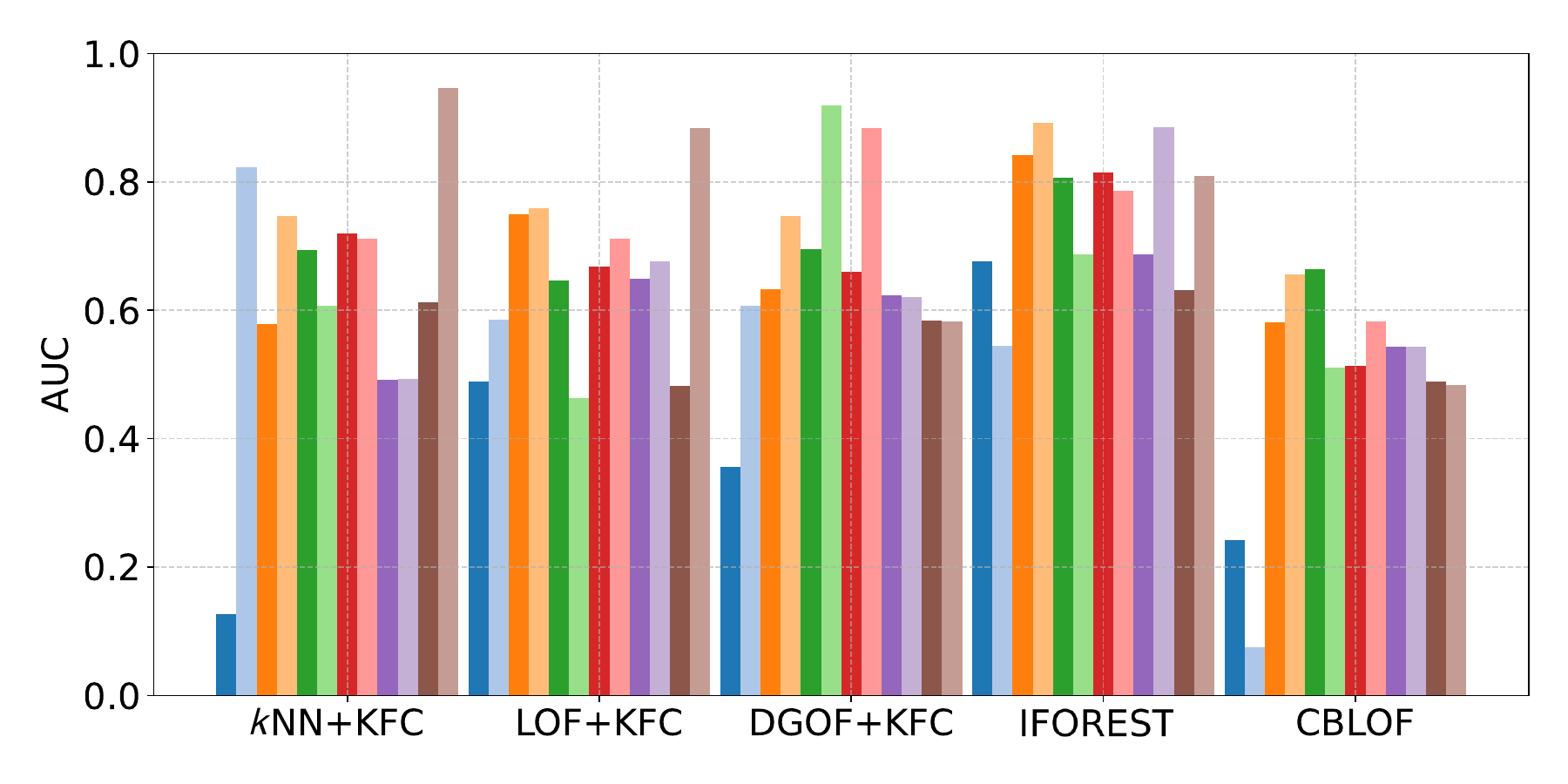}
	\caption{AUC comparison on D1-D12. D1 and D2 are with \textbf{only clusterliers}. The ten datasets D3-D12 are with \textbf{both scatterliers and clusterliers}.}
	\label{figure-zhu}
\end{figure}
\begin{table*}[!t]
  \centering
  \caption{AUC of 9 methods on 20 real benchmark datasets. The performance rank (the lower, the better) is shown in parentheses, and the best results are emphasized in \textbf{bold}.
  }
\resizebox{\linewidth}{!}{
    \begin{tabular}{l|ccccccccc}
    \toprule
     \textbf{Datasets}         & \textbf{COPOD} & \textbf{ECOD} & \textbf{OCSVM} & \textbf{IFOREST} & \textbf{$k$NN+KFC} & \textbf{LOF+KFC} & \textbf{DGOF+KFC} & \textbf{CBLOF} & \textbf{DROD}  \\ 
   &\textbf{\cite{COPOD}'2020} & \textbf{\cite{Li2022ECOD}'2022} &\textbf{\cite{OCSVMNEW}'2023} & \textbf{\cite{D-iforest}'2023} & \textbf{\cite{NA}'2024} & \textbf{\cite{NA}'2024} & \textbf{\cite{NA}'2024} & \textbf{\cite{CBLOFNEW}'2025} & \textbf{ours}\\
    \midrule
        Pageblocks   & 0.8754(6) & 0.9139(2) & 0.8903(5) & 0.8977(4) & 0.8061(8) & 0.7512(9) & 0.8413(7) & 0.9083(3) & \textbf{0.9154(1)}\\
        WPBC & 0.5233(4) & 0.4813(7) & 0.4743(8) & 0.4938(5) & 0.5323(2) & 0.4932(6) & \textbf{0.5600(1)} & 0.4686(9) & 0.5249(3) \\
        mnist & 0.7739(8) & 0.7463(9) & 0.8204(3) & 0.8008(5) & \textbf{0.8655(1)} & 0.8019(4) & 0.7829(7) & 0.7913(6) & 0.8603(2) \\
        musk & 0.9463(7) & 0.9559(6) & 0.8094(8) & 0.9996(4) &\textbf{1.0000(2)} & 0.8014(9) & \textbf{1.0000(2)} & \textbf{1.0000(2)} & 0.9888(5) \\
        Ionosphere  &0.7895(7) & 0.7284(9) &0.7395(8) & 0.8503(4) & 0.8300(6) &0.8889(2)   &0.8379(5) & \textbf{0.9013(1)} &0.8598(3) \\
        Waveform   & 0.7339(5) & 0.6035(8) & 0.5393(9) & 0.7071(7) & 0.7755(2) & 0.7557(4) & 0.7691(3) & 0.7212(6) & \textbf{0.8755(1)} \\
        %Cardiotocography简写为Cardiotoco.
        Cardiotoco. & 0.6629(7) & 0.7853(2) & \textbf{0.7872(1)} & 0.6940(4) & 0.6651(6) & 0.6439(8) & 0.6682(5) & 0.6278(9) & 0.7431(3) \\
        cardio & 0.9219(6) & 0.9350(2) & 0.9286(5) & 0.9299(4) & 0.8499(8) &0.6766(9) & 0.9322(3) & 0.8605(7) &\textbf{0.9368(1)} \\
        landsat & 0.4215(7) & 0.3678(8) & 0.3660(9) & 0.4872(6) &0.5929(2) & 0.5514(4) & 0.5553(3) & 0.5170(5) & \textbf{0.6290(1)} \\
        optdigits  & 0.6824(4) & 0.6045(5) & 0.5336(6) & 0.7140(3) & 0.5249(7) & 0.3921(8) & 0.2776(9) & 0.7547(2) & \textbf{0.9802(1)} \\
        pendigits  & 0.9048(5) & 0.9274(4) & 0.9354(3) &0.9481(2) & 0.8723(7) &0.4763(9) & 0.8305(8) & 0.8912(6) &\textbf{0.9537(1)} \\
        speech & 0.4911(3) & 0.4697(7) & 0.4639(8) & 0.4905(4) & 0.4773(5) & \textbf{0.6426(1)} & N/A(N/A) & 0.4723(6) & 0.5931(2) \\
        thyroid & 0.9393(7) & 0.9771(2) & 0.8437(8) & \textbf{0.9799(1)} & 0.9515(4) & 0.8075(9) & 0.9582(3) & 0.9428(5) & 0.9400(6) \\ 
        Pima & 0.6540(6) & 0.5944(8) & 0.6022(7) & 0.6752(4) & \textbf{0.7255(1)} & 0.5827(9) & 0.6967(3) & 0.6700(5) & 0.7215(2)  \\
        satellite & 0.6335(5) & 0.5830(8) & 0.5972(7) & 0.7076(3) & 0.6716(4) &0.5379(9) & 0.6136(6) & 0.7373(2) & \textbf{0.7873(1)}  \\ 
        satimage-2  & 0.9745(5) & 0.9649(6) & 0.9747(4) & 0.9928(3) & 0.9223(7) & 0.5894(9) & 0.6424(8) & \textbf{0.9989(1)} & 0.9940(2)\\ 
        vowels  & 0.4958(9) & 0.5929(7) &0.5507(8) &0.7600(6) & \textbf{0.9771(1)} & 0.9269(4) &0.9631(3) & 0.9214(5) &0.9748(2)  \\
        Seismic & 0.7343(2) & 0.6974(6) & 0.6516(8) & 0.7216(3) & \textbf{0.7381(1)} & 0.5579(9) & 0.6892(7) & 0.7190(4) & 0.7180(5) \\ 
        Banknote & \textbf{0.6556(1)} & 0.5220(4) & 0.5075(9) & 0.5139(8) & 0.5408(3) & 0.5146(7) & 0.5208(6) & 0.5950(2) & 0.5217(5) \\ 
        % HeartDisease 简写为HeartDis.
        HeartDis. &0.6946(2) & 0.6083(7) & 0.5491(9) & 0.6233(6) &0.6821(3) &0.6511(5) &\textbf{0.7151(1)} & 0.5895(8) &0.6677(4)\\
        \midrule  
        Avg. Rank & 5.30 & 5.85 & 6.65 & 4.30 & 4.00 & 6.70 & 4.95 & 4.70 & \textbf{2.50} \\ 
    \bottomrule
    \end{tabular}%
}
  \label{tab:auc}%
\end{table*}

\begin{table*}[t]
  \centering
  \caption{Precision-$s$ of 9 methods on 20 benchmark datasets. The performance rank (the lower, the better) is shown in parentheses, and the best results are emphasized in \textbf{bold}.
  }
\resizebox{\linewidth}{!}{
    \begin{tabular}{l|ccccccccc}
    \toprule
   \textbf{Datasets}& \textbf{COPOD} & \textbf{ECOD} & \textbf{OCSVM} & \textbf{IFOREST} & \textbf{$k$NN+KFC} & \textbf{LOF+KFC} & \textbf{DGOF+KFC} & \textbf{CBLOF} & \textbf{DROD}  \\ 
   &\textbf{\cite{COPOD}'2020} & \textbf{\cite{Li2022ECOD}'2022} &\textbf{\cite{OCSVMNEW}'2023} & \textbf{\cite{D-iforest}'2023} & \textbf{\cite{NA}'2024} & \textbf{\cite{NA}'2024} & \textbf{\cite{NA}'2024} & \textbf{\cite{CBLOFNEW}'2025} & \textbf{ours}\\
    \midrule
        Pageblocks & 0.3353(9) & 0.4314(4) & 0.3670(8) & 0.4186(6) & 0.4310(5) & 0.3892(7) & 0.4324(3) & \textbf{0.6078(1)} & 0.5235(2)\\ 
        WPBC & 0.2128(3) & 0.1277(9) & 0.1702(6.5) & 0.1553(8) & 0.1773(5) & 0.1809(4) & 0.2199(2) & 0.1702(6.5) & \textbf{0.2340(1)} \\ 
        mnist & 0.2357(8) & 0.1800(9) & 0.3710(3) & 0.2990(6) & 0.4210(2) & 0.2502(7)  & 0.3510(4) & 0.3329(5) & \textbf{0.5443(1)}  \\ 
        musk & 0.3608(6) & 0.4948(5) & 0.0223(9)  & 0.9495(2) & 0.5052(4) & 0.0464(7) & 0.0395(8) & \textbf{1.0000(1)} & 0.5361(3) \\ 
 
        Ionosphere & 0.5873(8) & 0.5317(9) & \textbf{0.8479(1)}  & 0.6706(7) & 0.7315(5) & 0.7659(4)  & 0.7262(6) & 0.8254(2) & 0.7698(3) \\ 
        Waveform & 0.0600(8) & 0.0400(9) & 0.0667(7)  & 0.2100(5) & 0.2150(4) & 0.1167(6) & 0.2550(2) & 0.2400(3) & \textbf{0.3900(1)}  \\ 
        Cardiotoco. & 0.3605(6) & \textbf{0.4957(1)}  & 0.2986(9) & 0.4043(4) & 0.3963(5) & 0.3119(7)  & 0.3094(8) & 0.4270(3) & 0.4442(2) \\ 
        cardio & \textbf{0.5284(1.5)} & \textbf{0.5284(1.5)} & 0.2595(8) & 0.5176(3) & 0.4138(6) & 0.2112(9) & 0.2879(7) & 0.5057(4) & 0.4943(5) \\ 
        landsat & 0.1800(8) & 0.1590(9) & 0.2532(5) & 0.2340(6) & 0.3003(2)& 0.2543(4)  & 0.2891(3) & 0.1845(7) & \textbf{0.3068(1)}  \\ 
        optdigits & 0.0133(7) & 0.0067(8) & 0.0422(3) & 0.0260(5) & 0.0322(4) & 0.0933(2)  & 0.0144(6) & 0.0000(9) & \textbf{0.7467(1)}  \\ 
        pendigits & 0.2628(3) & 0.3590(2) & 0.0748(7) & \textbf{0.3603(1)}  & 0.0855(6) & 0.0609(8)  & 0.0587(9) & 0.1154(5) & 0.1731(4) \\ 
        speech & 0.0328(4.5) & 0.0328(4.5) & \textbf{0.1065(1)}& 0.0213(7) & 0.0273(6) & 0.0492(3)   & N/A(N/A) & 0.0164(8) & 0.0755(2) \\ 
        thyroid & 0.1720(9) & 0.5484(2) & 0.2348(7) & \textbf{0.5817(1)}  & 0.2760(5) & 0.1774(8)  & 0.3208(4) & 0.2473(6) & 0.3763(3) \\ 
        
        Pima & 0.4888(6.5) & 0.4552(8) & 0.5479(2)& 0.5213(4) & \textbf{0.5485(1)} & 0.4447(9)  & 0.5062(5) & 0.4888(6.5) & 0.5448(3) \\ 
        satellite & 0.4804(5) & 0.4494(6) & 0.4202(8) & \textbf{0.5798(1)}  & 0.4951(4) & 0.3689(9)  & 0.4366(7) & 0.5693(3) & 0.5722(2)\\ 
        satimage-2 & 0.7465(3) & 0.6197(4) & 0.2629(8)  & 0.8704(2) & 0.5235(5) & 0.0892(9)  & 0.3122(7) & \textbf{0.9296(1)} & 0.4366(6) \\ 
        vowels & 0.0000(9) & 0.1800(7) & 0.6539(2)& 0.1600(8) & 0.4800(4) & 0.3333(5)  & 0.5300(3) & 0.3265(6) & \textbf{0.7600(1)}  \\ 
        Seismic & \textbf{0.2059(1)}  &0.1941(2) & 0.1559(6) & 0.1671(4) & 0.1637(5) & 0.0814(9)  & 0.1470(7) & 0.1294(8) & 0.1765(3) \\ 
        Banknote & \textbf{0.5443(1)}  & 0.4639(5) & 0.3956(9) & 0.4443(7) & 0.4676(4) & 0.4363(8)  & 0.4500(6) & 0.5131(2) & 0.4951(3) \\ 
        HeartDis. & 0.6083(2) & 0.5167(6) & 0.5292(4) & 0.5250(5) & 0.5306(3) & 0.4806(9)  & 0.4875(8) & 0.5083(7) & \textbf{0.6417(1)}  \\ 
        \midrule
        Avg. Rank & 5.43 & 5.55 & 5.68 & 4.60 & 4.25 & 6.70 & 5.53 & 4.70 & \textbf{2.40} \\   
    \bottomrule
    \end{tabular}%
}
  \label{tab:pr}%
\end{table*}

\begin{table}[t]
\centering
	\caption{The Wilcoxon signed-rank test results on real datasets using AUC as the metric.
 } % title of Table
    	\resizebox{1\linewidth}{!}{
	\begin{tabular}{l|lll|c} % centered columns (4 columns)
	\toprule
Counterparts & \textbf{R+} & \textbf{R-} & \textbf{p-value} & \textbf{null hypothesis} ($\alpha$=0.05)   \\
	\midrule
DROD vs  $k$NN+KFC   & 165    & 45       &0.0240     & \textbf{\textcolor[RGB]{0,200,50}{Reject}} \\
DROD vs  LOF+KFC   & 201     & 9       &0.0001   & \textbf{\textcolor[RGB]{0,200,50}{Reject}}      \\
DROD vs  DGOF+KFC  & 183     & 27       &0.0023      & \textbf{\textcolor[RGB]{0,200,50}{Reject}}   \\
DROD vs  IFOREST   & 191    & 19       &0.0006     & \textbf{\textcolor[RGB]{0,200,50}{Reject}}    \\
DROD vs  CBLOF   & 180     &30      &0.0037     & \textbf{\textcolor[RGB]{0,200,50}{Reject}}   \\
DROD vs  OCSVM  & 203    & 7       &0.0001        & \textbf{\textcolor[RGB]{0,200,50}{Reject}} \\
DROD vs  ECOD   & 192    & 18       &0.0005        & \textbf{\textcolor[RGB]{0,200,50}{Reject}}  \\
DROD vs  COPOD   & 185     & 25       &0.0017      & \textbf{\textcolor[RGB]{0,200,50}{Reject}}   \\
    \bottomrule
	\end{tabular}
	}
	\label{table:aucwilcox}
\end{table}

\subsubsection{Outlier Detection Performance Evaluation on Real Datasets}

Table~\ref{tab:auc} summarizes the AUC performance across 20 real-world datasets, and Table~\ref{tab:pr} provides complementary results in Precision-$s$. DROD achieves the best average ranking in both metrics. This indicates that its advantage is not limited to controlled synthetic settings, but also generalizes to diverse and noisy real data scenarios.
A closer examination shows that competing methods fluctuate more noticeably across datasets.
Their performance is affected by factors such as anomaly proportion, feature dimensionality, or cluster compactness. $k$NN+KFC and IFOREST perform competitively when anomalies exhibit clear boundary structures but degrade on datasets(e.g., landsat, WPBC, speech, etc.) where anomalies overlap with normal samples.
CBLOF and LOF variants underperform on scatterlier-dominated datasets(e.g., Waveform, WPBC, Cardiotoco., etc.), since they rely on assumptions related to cluster size or density concentration.
OCSVM and ECOD also show variable performance when the anomaly distribution is heterogeneous(e.g., satellite, satimage-2, cardio, etc.). DROD maintains stable top-tier ranking across most datasets. This consistency aligns with its dual reference mechanism. Local and global anomaly evidence are evaluated separately before being adaptively reweighted. The model does not depend on a single hypothesis about anomaly structure, which reduces the risk of failure when the true distribution differs from expectations. The Wilcoxon signed-rank results in Table~\ref{table:aucwilcox} further confirm that the performance improvement of DROD is statistically significant across all counterparts.

\subsection{Computation Efficiency Evaluation}\label{subsection-efficiency}

Two experiments are conducted to evaluate the computational efficiency of DROD, and the results are shown in Fig.~\ref{efficiency}. In the first experiment, the dimensionality is fixed at $d=36$, and the sample size $N$ increases from 1,000 to 6,000. In the second experiment, the sample size is fixed at $N=3,686$, and the dimensionality increases from 40 to 400 under a constant sampling iteration number $T=60$. Across both scalability settings, the runtime of DROD exhibits an approximately linear growth trend. This observation is consistent with its theoretical time complexity of $O(T \cdot N \cdot d \cdot \log N)$. 

In terms of dimensionality scalability, most compared methods show nearly linear runtime growth, while DROD grows at a notably slower rate. This is mainly because its candidate search and density estimation are performed on adaptively constructed local structures rather than full-dimensional dense representations. These results verify that the observed efficiency behavior is consistent with the algorithmic design. DROD achieves lower computational overhead in practice without sacrificing detection accuracy, making it suitable for large-scale and high-dimensional scenarios.

\subsection{Evaluation of NRS and Alternative Reference Sets}\label{subsection_nrs}
To comprehensively evaluate the proposed method, the NRS framework was compared with state-of-the-art techniques including GB \cite{GB}, HDIOD \cite{HDIOD}, and GNAN \cite{gnan}.
\begin{figure}[h]
	\centering
	\includegraphics[width=0.9\linewidth]{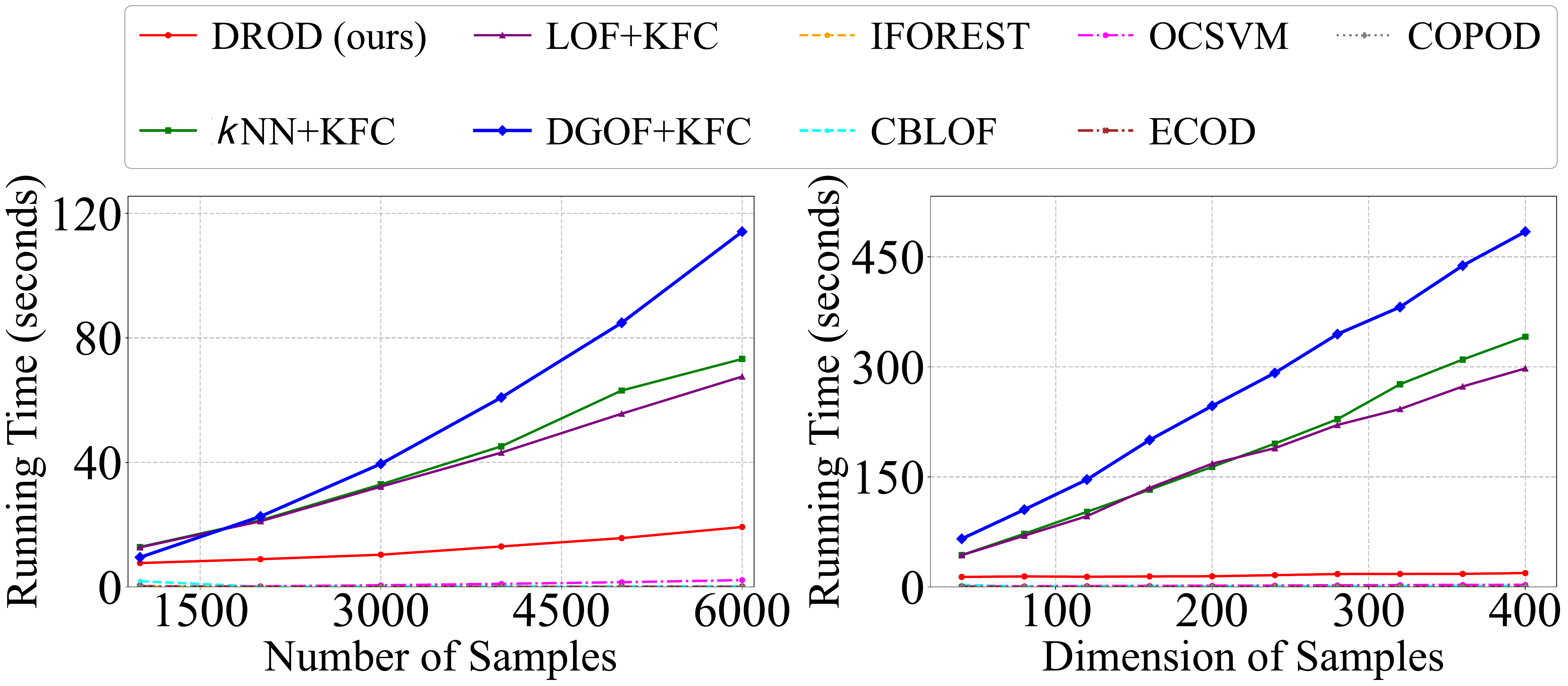}
	\caption{Runtime performance evaluation under two scalability scenarios: \textbf{Left}: varying the sample size from 1,000 to 6,000 with fixed dimensionality $d$=36; \textbf{Right}: varying the dimensionality from 40 to 400 with fixed sample size $n$=3,686.}
	\label{efficiency}
\end{figure}
The following insights highlight key comparative strengths and limitations:

As summarized in Table \ref{tab:auc_comparison_drod}, the NRS method achieves superior overall performance in anomaly detection by adaptively deriving reference subsets without hard partitioning or fixed neighborhood assumptions, thereby preserving local coherence and global relational structure; SAI can thus reflect genuine structural separation while LAI provides stable local density estimates. In comparison, GB relies on $K$-means with a fixed number of clusters for partitioning; when natural relations extend across cluster boundaries, such induced partitions may not fully respect cross-subset continuity, which can weaken global connectivity estimation. GNAN incorporates structural similarity, yet when local patterns become highly granular it may produce overly fine partitions, which tends to fragment higher-level relations and complicate macro-level anomaly scoring. HDIOD performs competitively in targeted settings but depends on manually specified $k$, so its effectiveness may vary under heterogeneous data distributions where the appropriate neighborhood size shifts across regions. Overall, these patterns suggest that methods grounded in a single partition scale or a fixed neighborhood assumption can be sensitive to anomaly morphology. However, the NRS framework jointly evaluates local (LAI) and global (SAI) evidence on adaptively formed subsets, enabling robust detection across both scatter-type and cluster-type anomalies with minimal parameter dependence.

\begin{table}[t]
  \centering
  \caption{AUC Results of Replacing DROD's Component with NRS and Other Competing Methods. The best results are highlighted in \textbf{bold} and the value of $k$ in HDIOD represents the parameter setting for the method.}
  \label{tab:auc_comparison_drod}
  \resizebox{\linewidth}{!}{% 自适应表格宽度，可根据需求删除以使用固定宽度
  \begin{tabular}{l|cccc}
    \toprule
      {Datasets} &   {NRS (ours)} &   {GB} &   {GNAN} &   {HDIOD} \\
    \midrule
      {Pageblocks} &   {\textbf{0.9154}} &   {0.4562} &   {0.5249} &   {0.8919 ($k$=6)} \\
      {WPBC} &   {0.5249} &   {0.4584} &   {0.4760} &   {\textbf{0.5942} ($k$=4)} \\
     {Ionosphere} &   {\textbf{0.8598}} &   {0.7271} &   {0.7586} &   {0.5867 ($k$=6)} \\
      {Cardiotocography} &   {0.7431} &   {0.5687} &   {0.5107} &   {\textbf{0.7831} ($k$=7)} \\
      {cardio} &   {0.9368} &   {0.5275} &   {0.6151} &   {\textbf{0.9407} ($k$=8)} \\
      {landsat} &   {\textbf{0.6290}} &   {0.4889} &   {0.4849} &   {0.6241 ($k$=3)} \\
      {optdigits} &   {0.9802} &   {0.5815} &   {0.5343} &   {\textbf{0.9977} ($k$=3)} \\
      {pendigits} &   {0.9537} &   {0.4454} &   {0.5384} &   {\textbf{0.9575} ($k$=7)} \\
      {speech} &   {\textbf{0.5931}} &   {0.3996} &   {0.4981} &   {0.4968 ($k$=8)} \\
      {Pima} &   {\textbf{0.7215}} &   {0.4805} &   {0.6043} &   {0.7160 ($k$=3)} \\
      {satellite} &   {0.7873} &   {0.4743} &   {0.4737} &   {\textbf{0.7953} ($k$=3)} \\
      {satimage-2} &   {\textbf{0.9940}} &   {0.4158} &   {0.5565} &   {0.9854($k$=3)} \\
      {vowels} &   {\textbf{0.9748}} &   {0.7897} &   {0.8996} &   {0.7295 ($k$=3)} \\
    \bottomrule
  \end{tabular}
  }
\end{table}

\subsection{  Validation of Potential Distance Metric}\label{subsection_distance}

The choice of distance metric significantly influences anomaly detection performance, particularly in data with diverse dimensional characteristics. This experiment evaluates three metrics within the DROD framework: Euclidean, Mahalanobis, and Chebyshev distances, comparing them across multiple datasets with varying dimensionality.

As summarized in Table \ref{tab:distance_metrics_comparison}, Euclidean distance achieves consistently strong performance across most datasets, particularly when the data dimensionality is moderate. This confirms that the neighborhood relations and subset structures constructed by DROD remain meaningful under standard geometric assumptions. Mahalanobis distance provides advantages on high-dimensional datasets such as Speech, where it adjusts for correlations among features and yields more discriminative separation. Chebyshev distance performs well in specific scenarios where the anomaly patterns manifest along a dominant coordinate axis, which causes local maxima in the coordinate-wise deviations.

The results demonstrate that Euclidean distance remains a reliable default choice, and the framework allows alternative metrics to be selected when aligned with the intrinsic properties of the data. This flexibility validates the generality of DROD across different data distributions.
\begin{table}[t]
\centering
\caption{Comparative AUC performance of DROD with different distance metrics on datasets of various scales. The best result for each dataset is highlighted in \textbf{bold}.}
\label{tab:distance_metrics_comparison}
\resizebox{\linewidth}{!}{
\begin{tabular}{l|c|ccc}
\toprule
  {\textbf{Datasets}} &   {\textbf{Scales (n $\times$ d)}}  &   {\textbf{Euclidean}} &   {\textbf{Mahalanobis}} &   {\textbf{Chebyshev}} \\

\midrule
  {Pageblocks} &   {5393 $\times$ 10} &   {\textbf{0.9154}} &   {0.9129} &   {0.9096} \\
  {WPBC} &   {198 $\times$ 33} &   {\textbf{0.5249}} &   {0.4999} &   {0.4841} \\
  {Ionosphere} &   {351 $\times$ 33} &   {0.8598} &   {\textbf{0.8602}} &   {0.8241} \\
  {Cardiotocography} &   {351 $\times$ 33} &   {0.7431} &   {0.6942} &   {\textbf{0.7829}} \\
  {Cardio} &   {1831 $\times$ 21} &   {0.9368} &   {0.9370} &   {\textbf{0.9407}} \\
  {Landsat} &   {6435 $\times$ 36} &   {\textbf{0.6290}} &   {0.4499} &   {0.6221} \\
  {Optdigits} &   {5216 $\times$ 64} &   {\textbf{0.9802}} &   {0.9149} &   {0.9436} \\
  {Pendigits} &   {6870 $\times$ 16} &   {\textbf{0.9537}} &   {0.9413} &   {0.9415} \\
  {Speech} &   {3686 $\times$ 400} &   {0.5931} &   {\textbf{0.6042}} &   {0.5016} \\
  {Pima} &   {768 $\times$ 8} &   {0.7215} &   {\textbf{0.7273}} &   {0.7209} \\
  {Satellite} &   {6435 $\times$ 36} &   {0.7873} &   {0.6691} &   {\textbf{0.7877}} \\
  {Satimage-2} &   {5803 $\times$ 36} &   {\textbf{0.9940}} &   {0.9610} &   {0.9921} \\
  {Vowels} &   {1456 $\times$ 12} &   {\textbf{0.9748}} &   {0.9593} &   {0.9717} \\
\bottomrule
\end{tabular}}
\end{table}
\subsection{Ablation Study}\label{subsection_ablation}

For validating the main technical components, DROD is compared with its three variants: DROD-L, DROD-S, DROD-0, and a density-based baseline method LOF that also computes the anomaly score locally within reference sets. DROD-L employs only the micro anomaly index LAI (defined by Eq.~\eqref{equ_LAS}) as the final anomaly score. DROD-S only utilizes the macro anomaly index SAI (computed according to Eq.~\eqref{equ_SAS}) as the final anomaly score. DROD-0 is a variant of DROD without utilizing the sampling enhancement strategy and it has been presented in Section~\ref{sct:detection}. Comparative results on all 20 real datasets are presented in Table~\ref{tab:ablation}. 

DROD outperforms all its alternative variants and the LOF baseline on the majority of datasets. This demonstrates that neither local nor global anomaly information alone is sufficient to ensure stable detection across heterogeneous data distributions. DROD-L falls short on datasets where globally isolated anomaly clusters appear locally dense. DROD-S underperforms when anomalies are dispersed in small, sparse subsets that do not immediately exhibit strong global separation. These observations confirm that local LAI and global SAI capture complementary structures. DROD also improves over DROD-0 on most datasets, indicating that sampling enhancement contributes positively by improving the chance of isolating anomalous samples in multiple views. The sampling-based aggregation further increases robustness against local noise or structural overlap. These results validate that the complete DAI formulation, which integrates LAI, SAI, and sampling enhancement, is necessary to achieve consistently high detection performance.
\begin{table}[t]
  \centering
  \caption{AUC performance comparison of baseline and three ablated versions of DROD. The best performance is marked in \textbf{bold}.}
\resizebox{1\linewidth}{!}{
    \begin{tabular}{l|ccccc}
    \toprule
    \textbf{Datasets}   & \textbf{LOF+KFC}  & \textbf{DROD-L} & \textbf{DROD-S} & \textbf{DROD-0}  & \textbf{DROD}   \\ 
    \midrule
        Pageblocks      & 0.7263    &0.7250     &0.9020     &0.9116  & \textbf{0.9154} \\ 
        WPBC            & 0.4932    &0.4973    &0.5545     &\textbf{0.5546}  &0.5249  \\ 
        mnist           & 0.8019   &0.4361    &0.7863     &0.7599 &\textbf{0.8603}  \\
        musk            & 0.8014    &0.7270   &0.9604     &0.9481 &\textbf{0.9888}  \\
        Ionosphere      &\textbf{0.8889}     &0.7543     &0.8702     &0.8504    &0.8598\\
        Waveform        & 0.7826    &0.3864     &0.8220     &0.8182  & \textbf{0.8755}  \\
        %% Cardiotocography进行了缩写Cardiotoco.
        Cardiotoco.  & 0.5915  &0.5498     &0.7176 &0.7111  & \textbf{0.7431} \\
        cardio          &0.5887     &0.7597    &0.8965     &0.9127  &\textbf{0.9368} \\
        landsat         & 0.5512    &0.5092     &0.6046     &0.6104  &\textbf{0.6290}  \\ 
        optdigits       &0.3930     &0.5466     &0.9728     &0.9503   & \textbf{0.9802}  \\
        pendigits       &0.4763     &0.4794     &\textbf{0.9628}     &0.9335   & 0.9537 \\ 
        speech          &0.5219    &\textbf{0.6000}     &0.5797     &0.5913   &0.5931 \\
        thyroid         &0.8075    &0.9202    &0.8618     &0.9009   &\textbf{0.9400} \\
        Pima            & 0.5396    &0.6062     &0.6914     &0.6980   & \textbf{0.7215}  \\ 
        satellite       & 0.5179    &0.5483    &0.7790     &0.7625   &\textbf{0.7873}  \\ 
        satimage-2      & 0.5989    &0.7573     &0.9925     &0.9908   &  \textbf{0.9940} \\  
        vowels          & 0.8684    &0.8845     &0.8580     &0.9166   &\textbf{0.9748}   \\
        Seismic         &0.5579    &0.5486     &0.7053     &0.7015   &\textbf{0.7180} \\
        Banknote        & 0.5146   &0.4515     &0.5143    &0.5013   &\textbf{0.5217}   \\
        HeartDis.    &0.6511     &0.6127     &0.5928    &0.6229    &\textbf{0.6677}\\       
        \midrule  
        Avg. AUC &0.6336 &0.6150  &0.7812  & 0.7823 &\textbf{0.8093 } \\ 
    \bottomrule
    \end{tabular}%
}
  \label{tab:ablation}%
\end{table}
\subsection{Downstream Clustering Task Enhancement}\label{subsection_downstream}

To validate the effectiveness of the proposed method for downstream clustering tasks, DROD and the aforementioned comparison methods are employed to remove the top 150 outlier candidates from the real benchmark dataset ``optdigits'' \cite{Agg2015}. The clustering performance of $K$-means \cite{kmeans} is subsequently recorded. The ``optdigits'' dataset contains normal data representing digits 1–9 (nine natural classes) along with 150 outliers. Accordingly, the number of clusters for $K$-means is set to 9. Additionally, as a control experiment for evaluating dataset pre-processing without outlier detection methods, a detector labeled ``Random'' is included, which corresponds to randomly removing 150 samples from the dataset.
\begin{table}[t]
\centering
\caption{DBI of $K$-means on ``optdigits'' dataset with outlier removal using the nine compared methods as detectors. The AUC performance x for reference, and the symbol ``$\Delta$'' denotes boosted DBI based on the DBI performance without outlier removal. Results indicating better or worse performance are marked in \textcolor[RGB]{0,200,50}{green} or \textcolor{red}{red}, respectively.}
\resizebox{0.8\linewidth}{!}{
	\begin{tabular}{l|c|c} % centered columns (4 columns)
	\toprule
\textbf{Detectors} & \textbf{AUC} & \textbf{DBI ($\Delta$)}   \\
	\midrule
Random   & -    & 1.9578 \;(0.0000)\\
$k$NN+KFC   & 0.5249    & 1.9132\; (\textcolor[RGB]{0,200,50}{-0.0446})\\
LOF+KFC   & 0.3921        & 1.9227 \;(\textcolor[RGB]{0,200,50}{-0.0351})  \\ %1.9227
DGOF+KFC  & 0.2831          & 1.9117 \;(\textcolor[RGB]{0,200,50}{-0.0461})   \\
IFOREST   & 0.7140        & 1.9341 \;(\textcolor[RGB]{0,200,50}{-0.0237})    \\
CBLOF   & 0.7547         & \ 2.0125 \;(\textcolor{red}{+0.0547})   \\
OCSVM   &  0.5336              &1.9573 (\textcolor[RGB]{0,200,50}{-0.0005})   \\     
ECOD   & 0.6045           & 1.9394 \;(\textcolor[RGB]{0,200,50}{-0.0184})  \\
COPOD   & 0.6824        & 1.9399 \;(\textcolor[RGB]{0,200,50}{-0.0179})   \\
DROD (ours)     & 0.9802         & 1.9067 \;(\textcolor[RGB]{0,200,50}{-0.0511}) \\
    \bottomrule
	\end{tabular}
	}
	\label{tab:clustering}
\end{table}

% Table~\ref{tab:clustering} reports the result of the DBI of the K-means on different outlier-removed ``optdigits'' datasets preprocessed by different approaches. From the table, almost all methods (except CBLOF) enhance K-means performance compared to the performance on the original dataset without outlier removal (i.e., the performance of ``Random''). CBLOF results in a negative effect in enhancing clustering due to its clusterlier hypothesis. When the majority of outliers of a dataset tend to be scatterliers, the performance of CBLOF may surely degrade in outlier detection. Clearly, DROD outperforms all counterparts in enhancing downstream clustering performance, as it is competent in detecting and removing the various outliers in complex real datasets.
Table~\ref{tab:clustering} shows that although most methods improve clustering quality compared with ``Random'', CBLOF underperforms when anomaly clusters are prominent yet insufficiently separated in size or distance from large normal clusters. In such cases, its size-based assumption that outliers appear as compact micro clusters can be violated, causing it to under-detect anomalies or mistakenly remove boundary-normal samples. This behavior is evident on the optdigits dataset, where compact anomalous groups are comparable in size to or located close to major normal clusters. In contrast, DROD achieves the lowest DBI, indicating that it effectively eliminates structurally disruptive samples that interfere with cluster formation. These results confirm the effectiveness of DROD in boosting clustering performance.

\subsection{Hyper-parameter Evaluation}\label{subsection_parameter}

A sampling enhancement mechanism is designed to make the anomaly score computation of DROD more effective in exploring outliers. The two hyper-parameters involved in this mechanism, i.e., sampling rate $\eta$ and sampling times $T$, are evaluated to show the rationality of our settings in all the above experiments and the hyper-parameter-insensitivity of DROD.

\begin{figure}[t]
	\centering
	\includegraphics[width=0.9\linewidth]{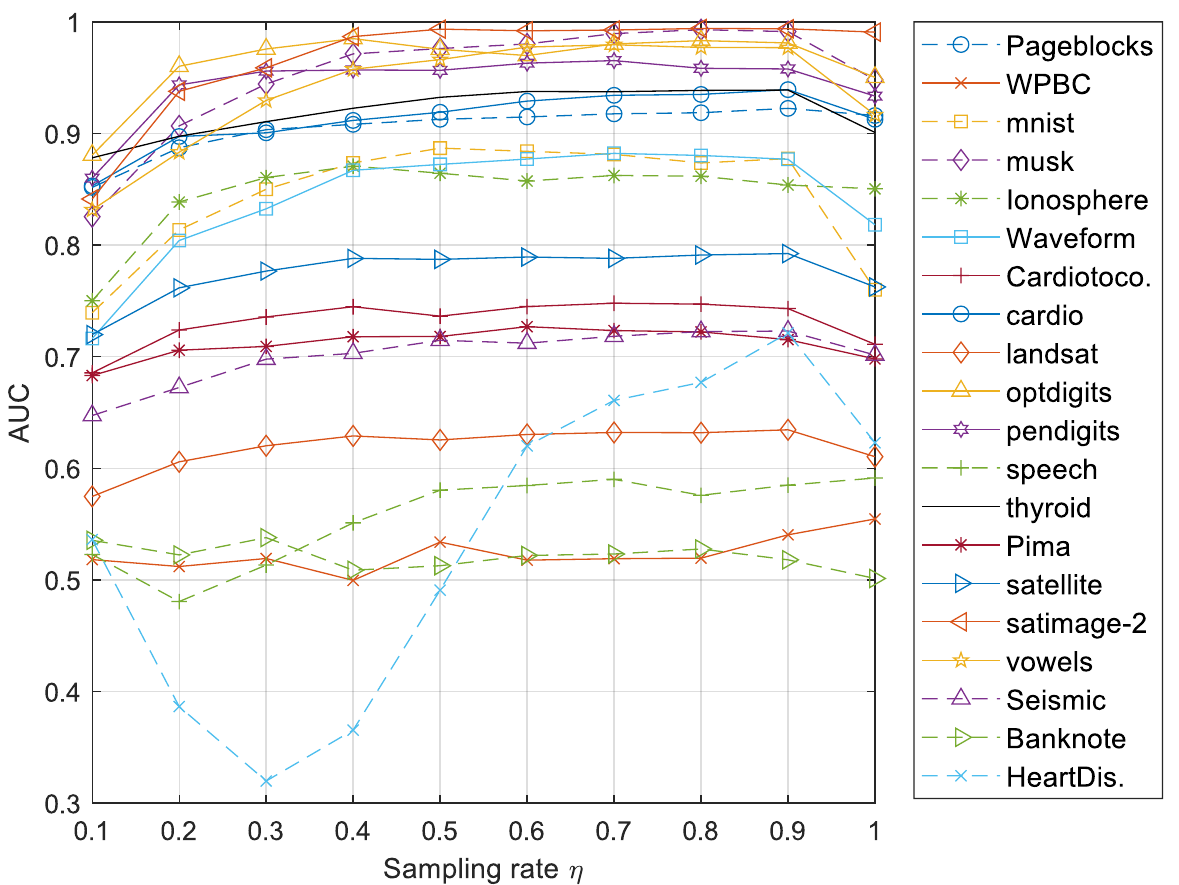}
	\caption{AUC of DROD on 20 datasets with sampling rate $\eta$ from $[0.1, 1]$ with step size 0.1 under sampling times $T=100$.}
	\label{figure-rate}
\end{figure}

\begin{figure}[t]
	\centering
	\includegraphics[width=0.90\linewidth]{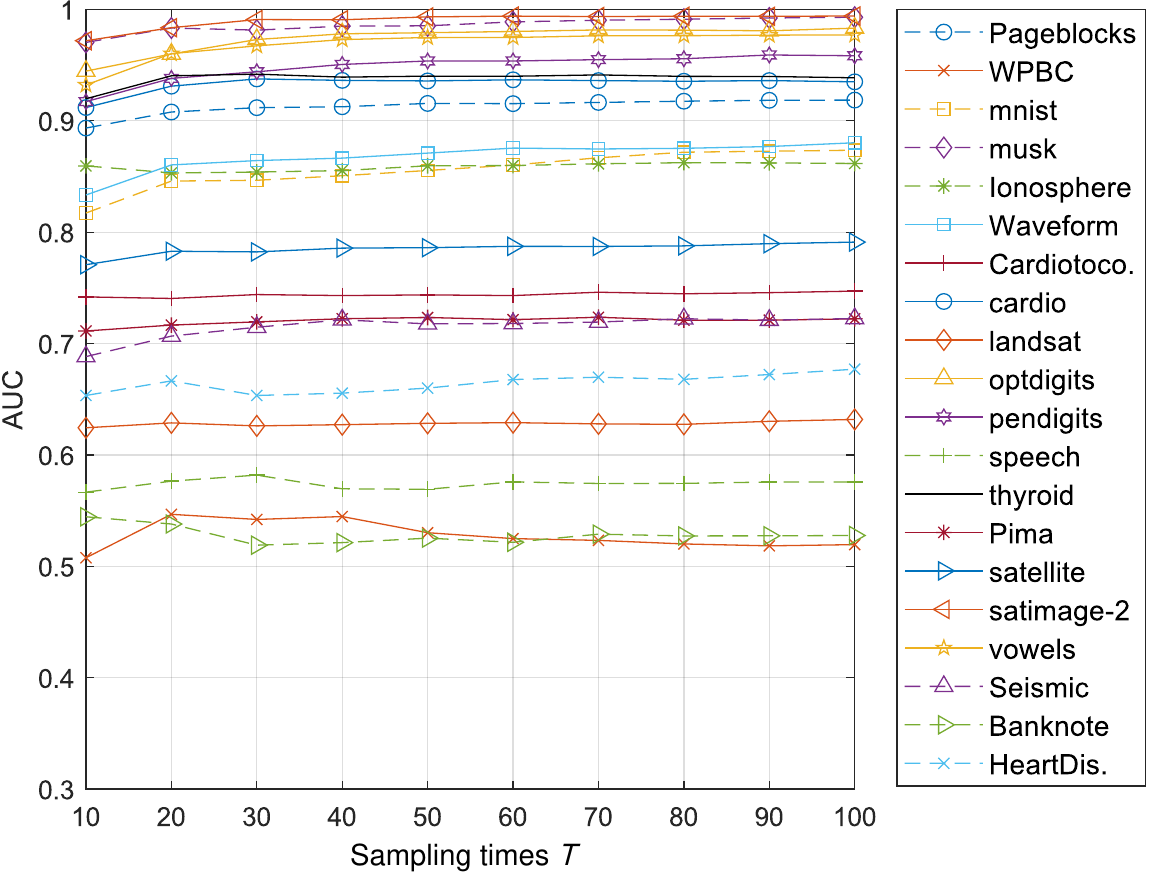}
	\caption{AUC of DROD on 20 datasets with sampling times $T$ from $[10, 100]$ with step size 10 under sampling rate $\eta=0.8$.}
	\label{figure-time}
\end{figure}

\subsubsection{Sensitivity to Sampling Rate $\eta$}

In the above experiment, the random sampling rate is set at $\eta=0.8$. This experiment is designed to explore parameter settings by conducting 100 random samplings on 20 real datasets. The sampling rate $\eta$ is increased from 0.1 to 1 with a step size of 0.1, and the corresponding AUC performance is recorded in Fig.~\ref{figure-rate} to evaluate the sensitivity of DROD to different $\eta$ settings. The results indicate that the AUC performance remains insensitive to the sampling rate $\eta$ across a wide range from 0.4 to 0.9. When the sampling rate increases from 0.9 to 1, the AUC decreases on 18 datasets. In the case of $\eta=1$, the sampling enhancement is effectively omitted, leading to a slight decrease in detection performance. When $\eta$ is relatively small, ranging from 0.1 to 0.4, the sampled subset may fail to represent the comprehensive distribution of the original dataset due to an insufficient number of samples. This also explains why datasets with relatively few samples or a high proportion of outliers, such as WPBC, HeartDis., and Banknote, exhibit higher fluctuation in AUC curves. Therefore, the recommended range for the sampling rate is between 0.4 and 0.9 for larger-scale datasets with a relatively small proportion of outliers. Consequently, setting the sampling rate to $\eta=0.8$ can be considered a reasonable choice for most experimental datasets.

\subsubsection{Sensitivity to Sampling Times $T$}

In the above experiments, the sampling times were fixed at $T=60$. To assess its influence, we varied $T$ from 10 to 100 with a fixed sampling rate $\eta=0.8$. As shown in Fig.~\ref{figure-time}, AUC generally increases with larger $T$, but saturates beyond $T=60$, indicating limited further gain relative to computation cost. Therefore, $T=60$ is adopted as a practical trade-off.
\begin{figure}[!t]
  \centering
  \subfloat[Samples vs Anomaly\label{fig:subfig_a}]{
    \includegraphics[width=.45\linewidth]{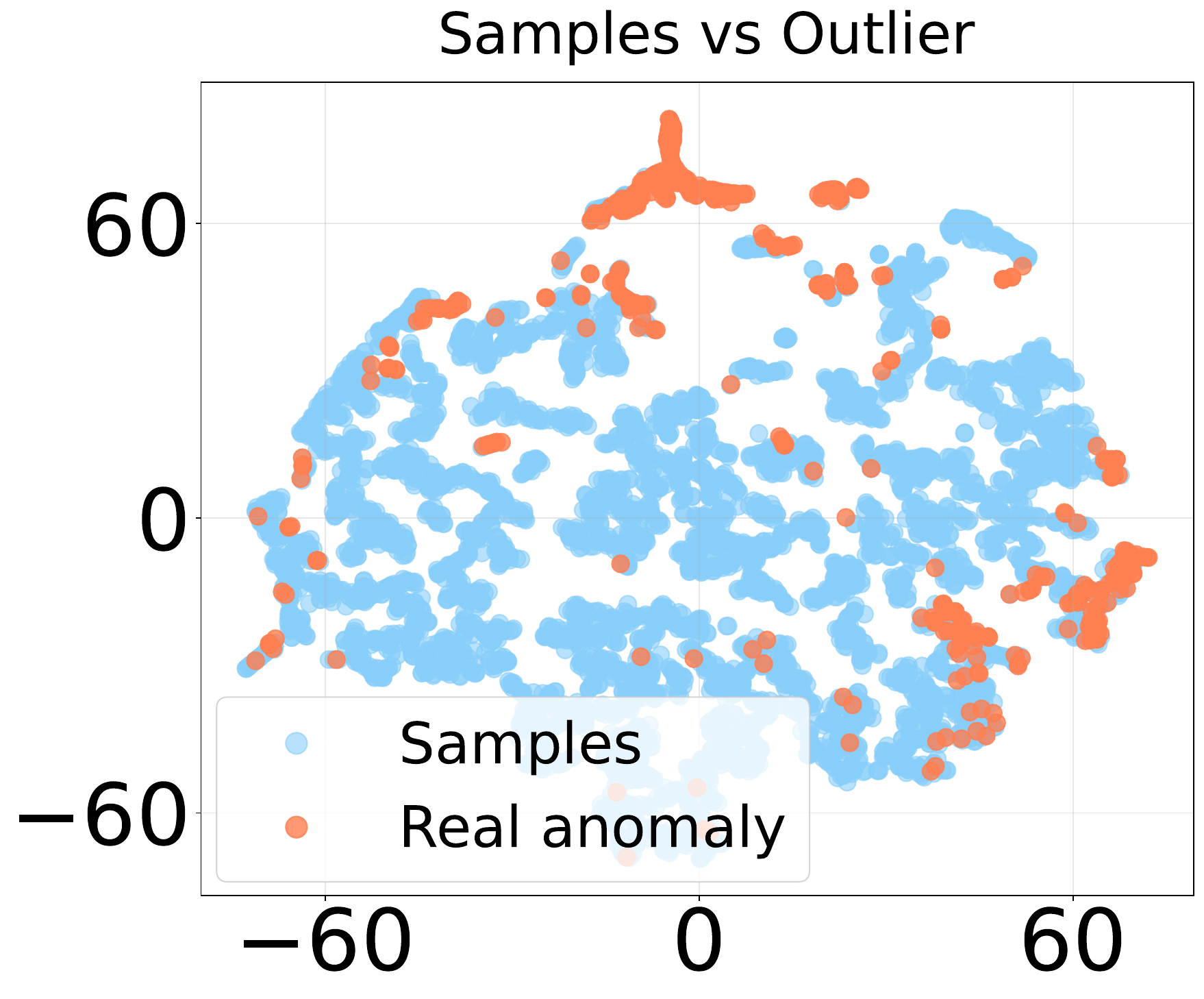}
  }\hfill
  \subfloat[DAI values of anomaly\label{fig:subfig_b}]{
    \includegraphics[width=.45\linewidth]{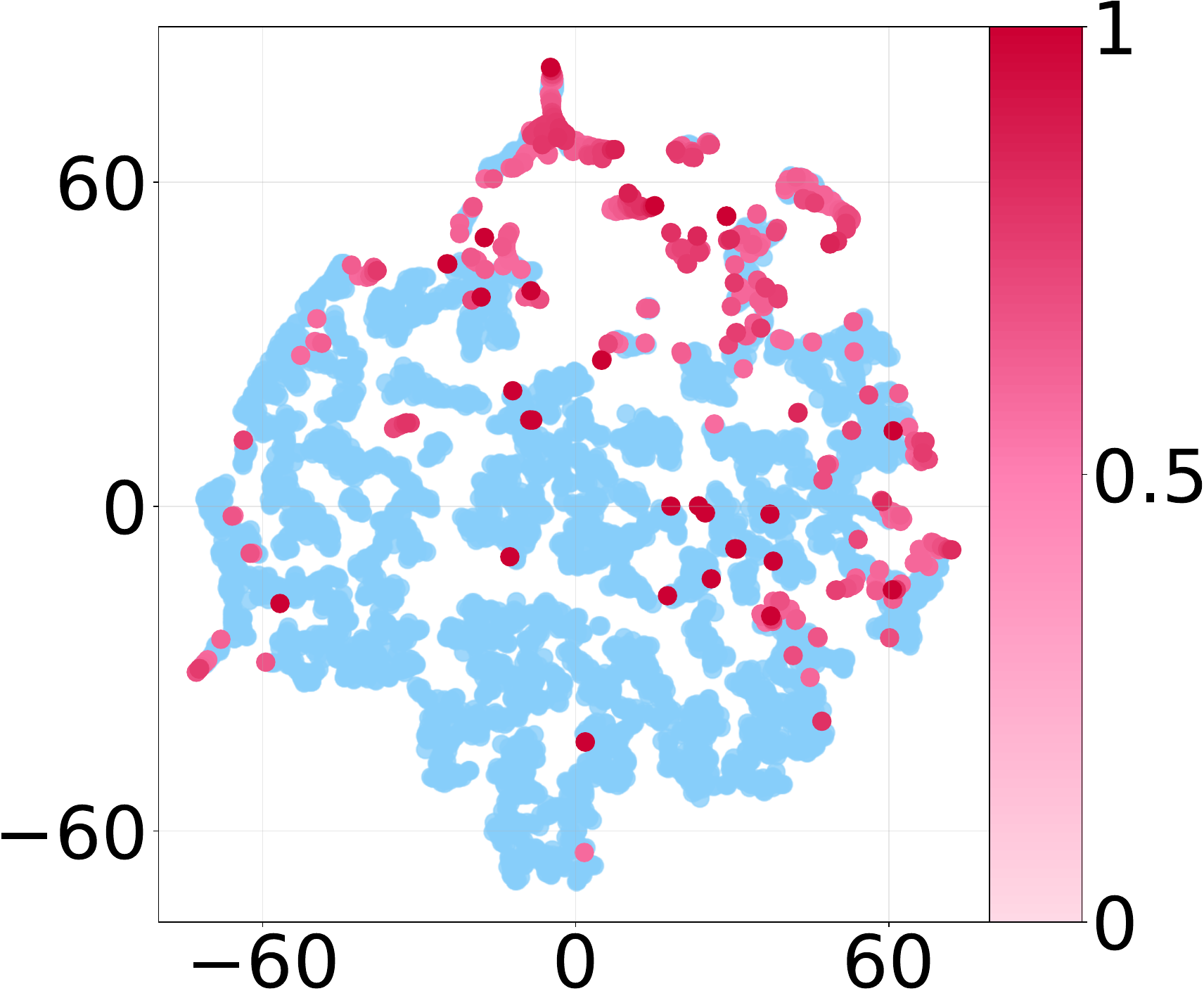}
  }

  \subfloat[LAI values of anomaly\label{fig:subfig_c}]{
    \includegraphics[width=.45\linewidth]{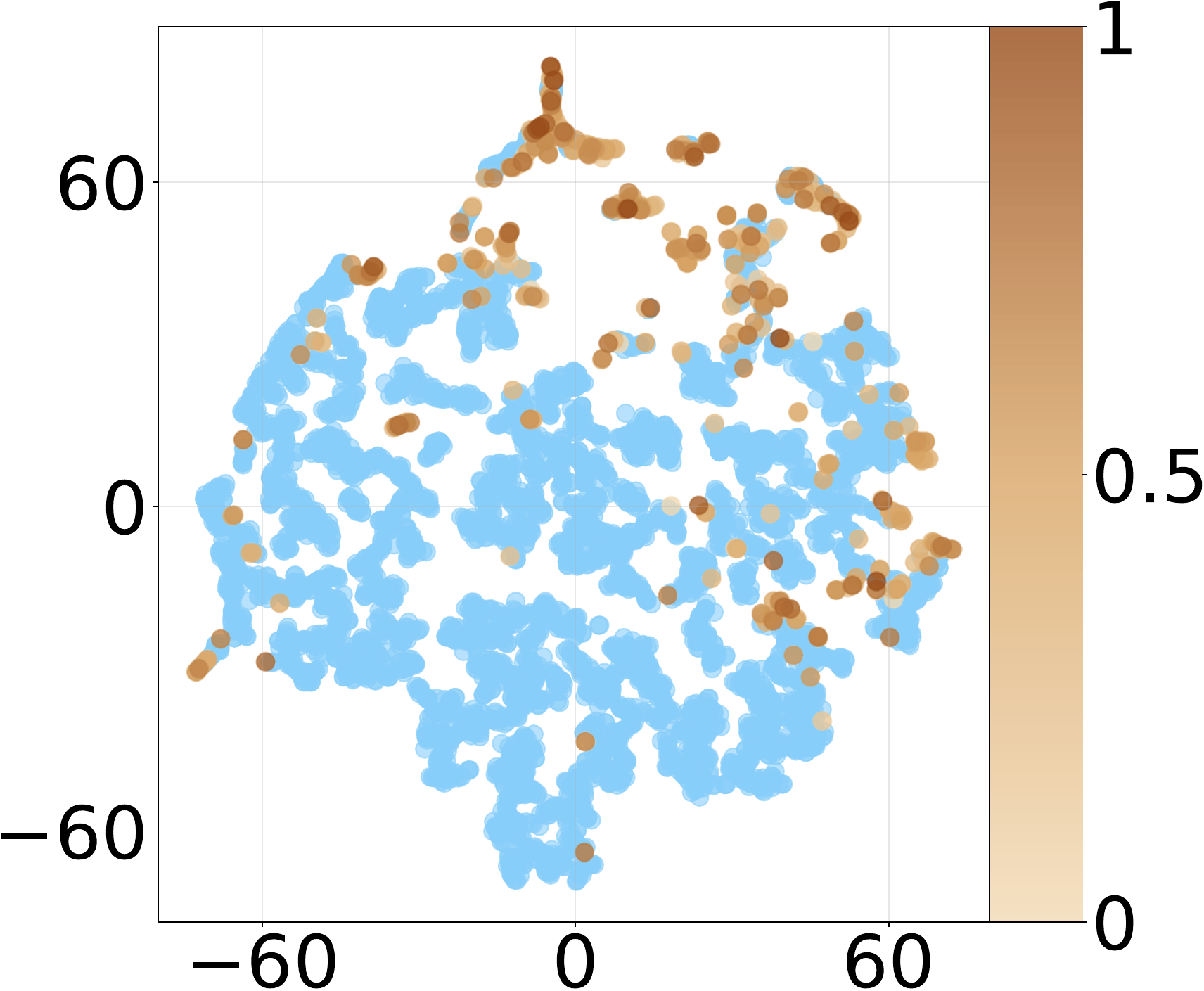}
  }\hfill
  \subfloat[SAI values of anomaly\label{fig:subfig_d}]{
    \includegraphics[width=.45\linewidth]{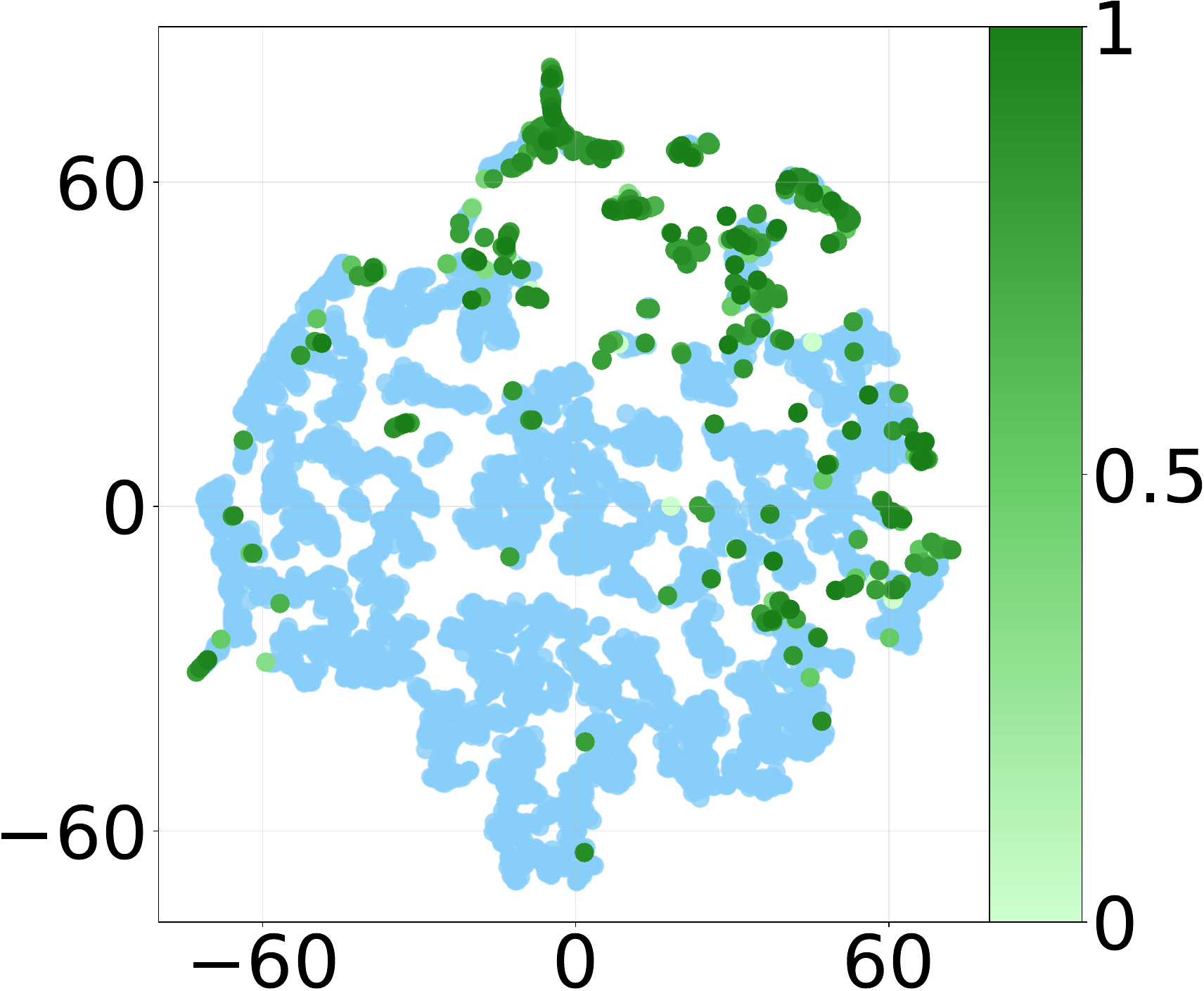}
  }
  \caption{Visualization of DAI mechanism through t-SNE embedding: (a) Distribution of normal data and ground-truth anomalies; (b) DAI values of predicted anomalies; (c) LAI values of predicted anomalies; (d) SAI values of predicted anomalies, demonstrating the synergistic relationship between local and global anomaly indicators.}
  \label{laisai}
\end{figure}

\subsection{Study of the Relationships among DAI, SAI, and LAI}\label{subsection_daivalidity}

To demonstrate the computational mechanism and effectiveness of DAI, a t-SNE visualization analysis is conducted on the dataset, as shown in Fig. \ref{laisai}.

Subfigure (a) distinguishes between normal and anomalous samples based on ground-truth labels. It can be observed that the anomalous samples exhibit a certain degree of clustering, forming distinct regional separations from the normal samples, which offers visual justification for anomaly detection at the distribution level. Subfigure (b) displays the distribution of DAI values for all predicted anomalies, facilitating the observation of subsequent interactions between LAI and SAI values. 

Subfigures (c) and (d) respectively show the performance of LAI and SAI for samples predicted as anomalous. The visualization reveals a synergistic relationship between LAI and SAI: when SAI is low, even samples with high LAI may have their resulting DAI suppressed to a lower level; conversely, when SAI is high, the anomaly degree of samples with high LAI is further amplified, leading to a higher DAI score. This phenomenon fully demonstrates the role of the weighting mechanism, where \(\beta(s_m) = \text{SAI}(s_m)\), in distinguishing noise from true anomalies. Furthermore, the visualization indicates that certain types of anomalies, such as cluster-based outliers, are more easily highlighted during computation because of their weak connectivity in the reference sets, resulting in higher DAI prominence. These observed distribution characteristics align well with the original design intention of DAI, thereby further verifying its practical effectiveness.

\section{Data and code availability}

All datasets used in this work can be found at \url{https://github.com/gordonlok/DROD} under the ``data'' file, and the datasets' origins are cited in Table~\ref{table:datasets}.

\section{Conclusion}\label{sec_con}

In this paper, we propose a novel unsupervised outlier detection method called DROD to distinguish between scatterliers and clusterliers. It effectively alleviates the masking effect brought by clusterliers in detecting scatterliers through the hierarchically dual reference set design, which reflects clusterliers in a macro subset reference set level and finely indicates scatterliers within subset reference sets. To the best of our knowledge, this is the first attempt to simultaneously tackle the scatterlier and clusterlier detection problems with the considering their couplings. By comparing with the state-of-the-art related methods on various datasets, the proposed DROD method is proven to be superior in terms of both accuracy and robustness. Extensive experiments, including outlier detection performance comparison, significance tests, ablation studies, downstream clustering task enhancement evaluation, and sensitivity evaluation of the hyper-parameters, collectively demonstrate the efficacy of DROD.

Although DROD demonstrates superiority, it is also not exempt from limitations. When dealing with datasets with highly imbalanced clusters, distinguishing the micro-clusters composed of clusterlier samples from normal small clusters is still an under-explored but significant issue. For such a thorny problem, connecting the anomaly score measurement with the downstream analysis tasks would be a promising future research orientation of unsupervised outlier detection.
\bibliographystyle{IEEEtran}    % 官方IEEE样式
\bibliography{references}  
% \iffalse

\begin{IEEEbiography}[{\includegraphics[width=1in,height=1.5in,keepaspectratio]{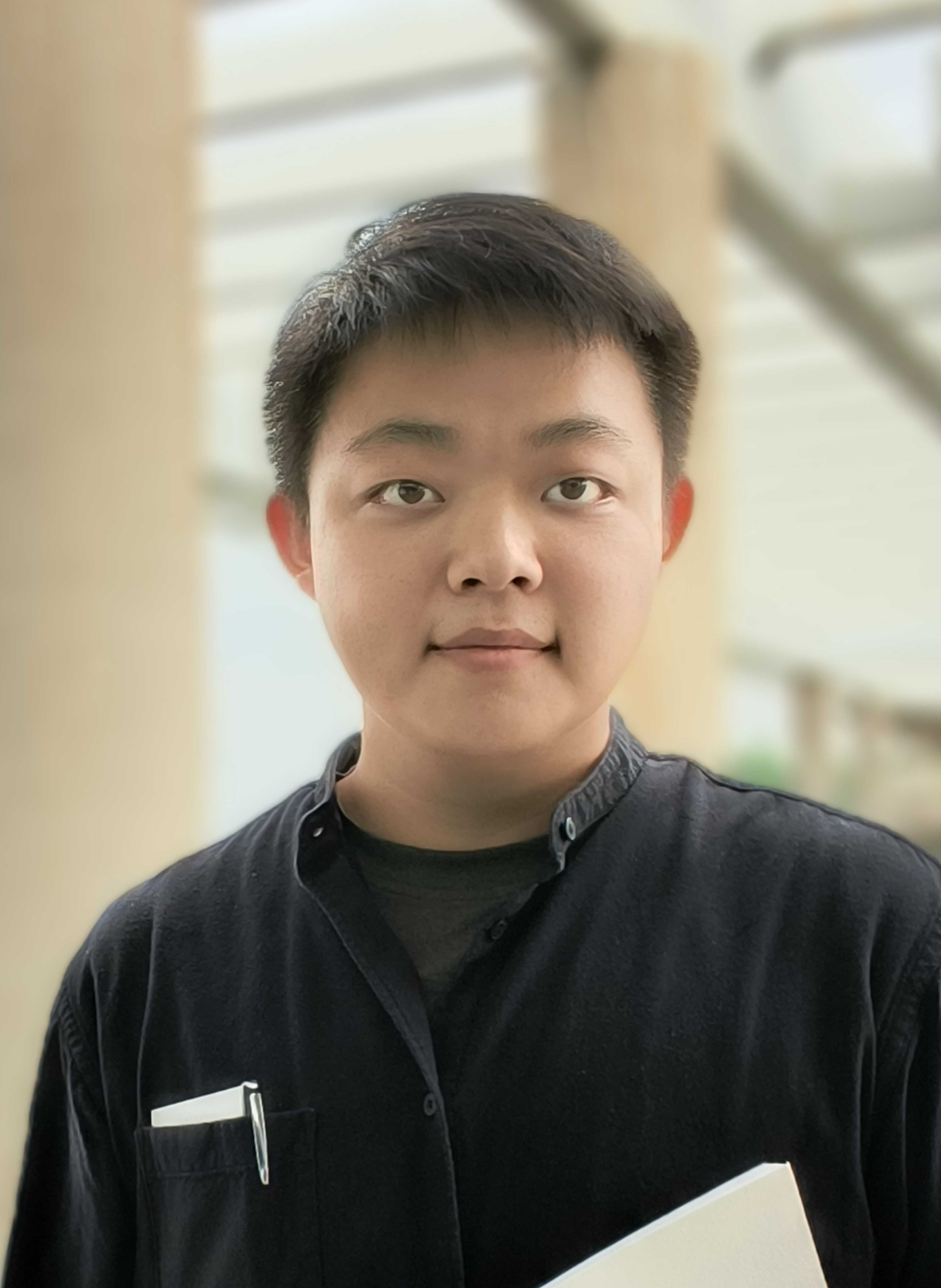}}]{Yiqun Zhang}
(Senior Member, IEEE) received the B.Eng. degree from South China University of Technology, Guangzhou, China, in 2013, and the M.S. and Ph.D. degrees from Hong Kong Baptist University (HKBU), Hong Kong, China, in 2014 and 2019, respectively. He is currently with the School of Computer Science and Technology of Guangdong University of Technology, Guangzhou, China, and also with the Department of Computer Science of HKBU. His research works have been published in top-tier journals and conferences, including TPAMI, TIP, TNNLS, TCYB, SIGMOD, SIGKDD, NeurIPS, CVPR, to name a few. His current research interests include machine learning, data science, and their applications. Dr. Zhang serves as an Associate Editor for the \textit{IEEE Transactions on Emerging Topics in Computational Intelligence}. 
\end{IEEEbiography}

\begin{IEEEbiography}
[{\includegraphics[width=1in,height=1.5in,keepaspectratio]{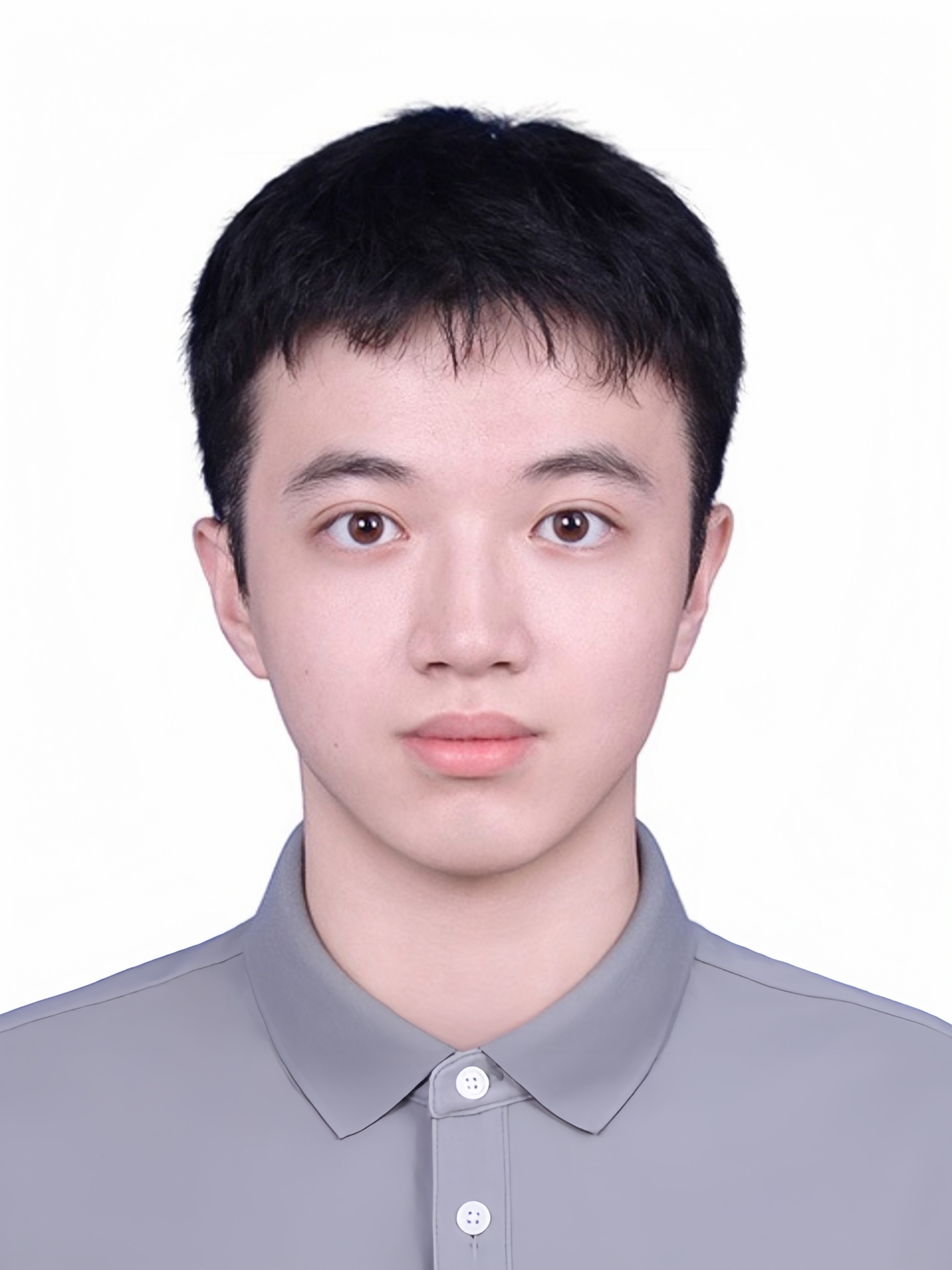}}]
{Zexi Tan} is with the School of Computer Science and Technology, Guangdong University of Technology, Guangzhou, China. His current research interests include unsupervised machine learning and time-series data analysis. He has published a series of research works in reputable journals and conferences, including TNNLS, AAAI'26, BIBM'25, and so on. 
%Mr. Tan serves as a reviewer for several international journals and conferences, including TETCI, PRICAI'25, BIBM'24, BIBM'25, and so on.
\end{IEEEbiography}

\begin{IEEEbiography}[{\includegraphics[width=1in,height=1.5in,keepaspectratio]{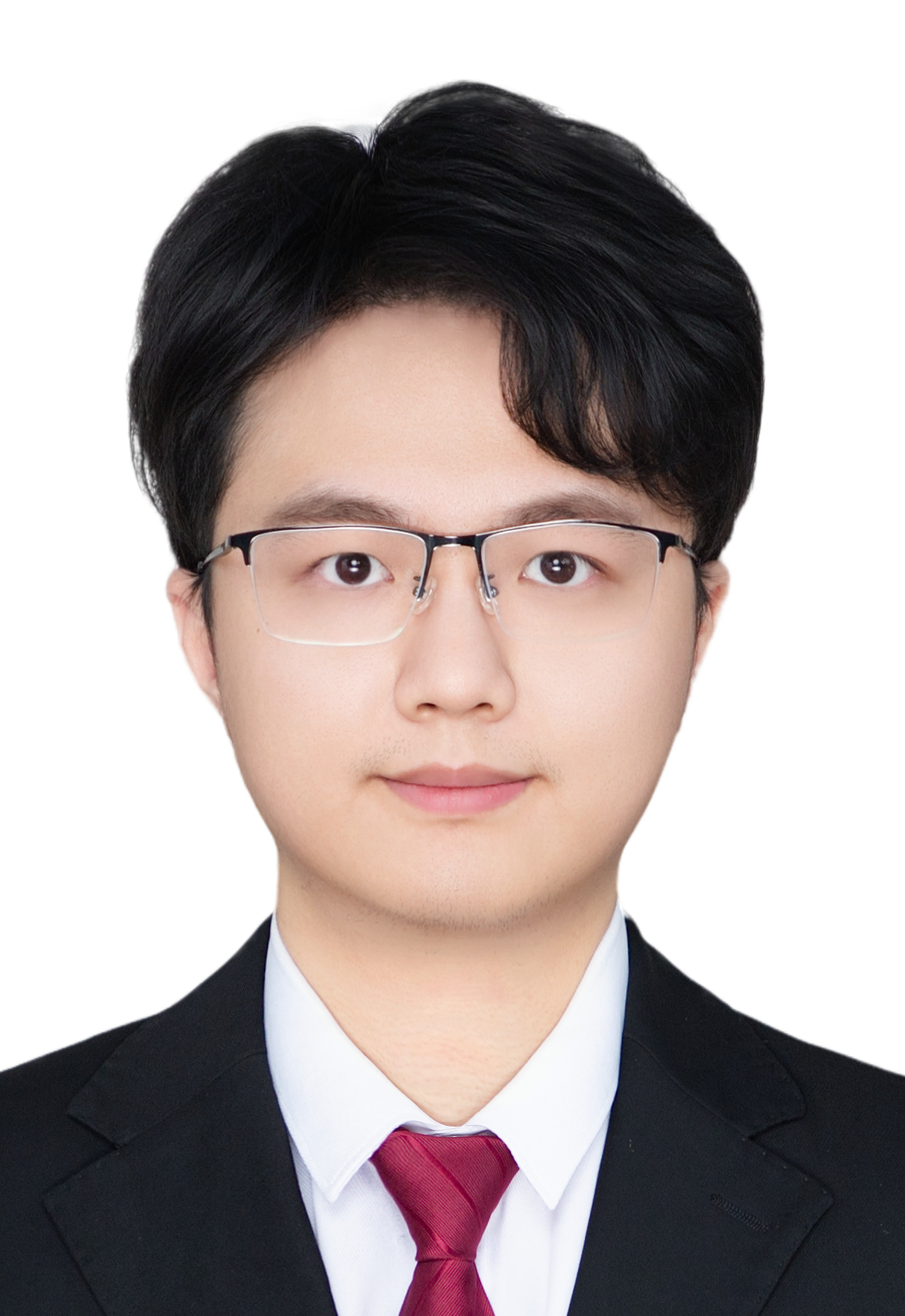}}]{Xiaopeng Luo}(Member, IEEE)
received the B.Eng. degree from Shenzhen University, Shenzhen, China, in 2019, the M.S. degree from Hong Kong Baptist University, Hong Kong, China, in 2021, and the MPhil degree from Beijing Normal University - Hong Kong Baptist University United International College (UIC), Zhuhai, China, in 2023. His current research interests include machine learning, differential privacy, and their applications. He has published a series of research works in reputable journals and conferences, including TNNLS, AAAI'26, ICDCS'24, BIBM'24, and so on. 
\end{IEEEbiography}

\begin{IEEEbiography}
[{\includegraphics[width=1in,height=1.5in,keepaspectratio]{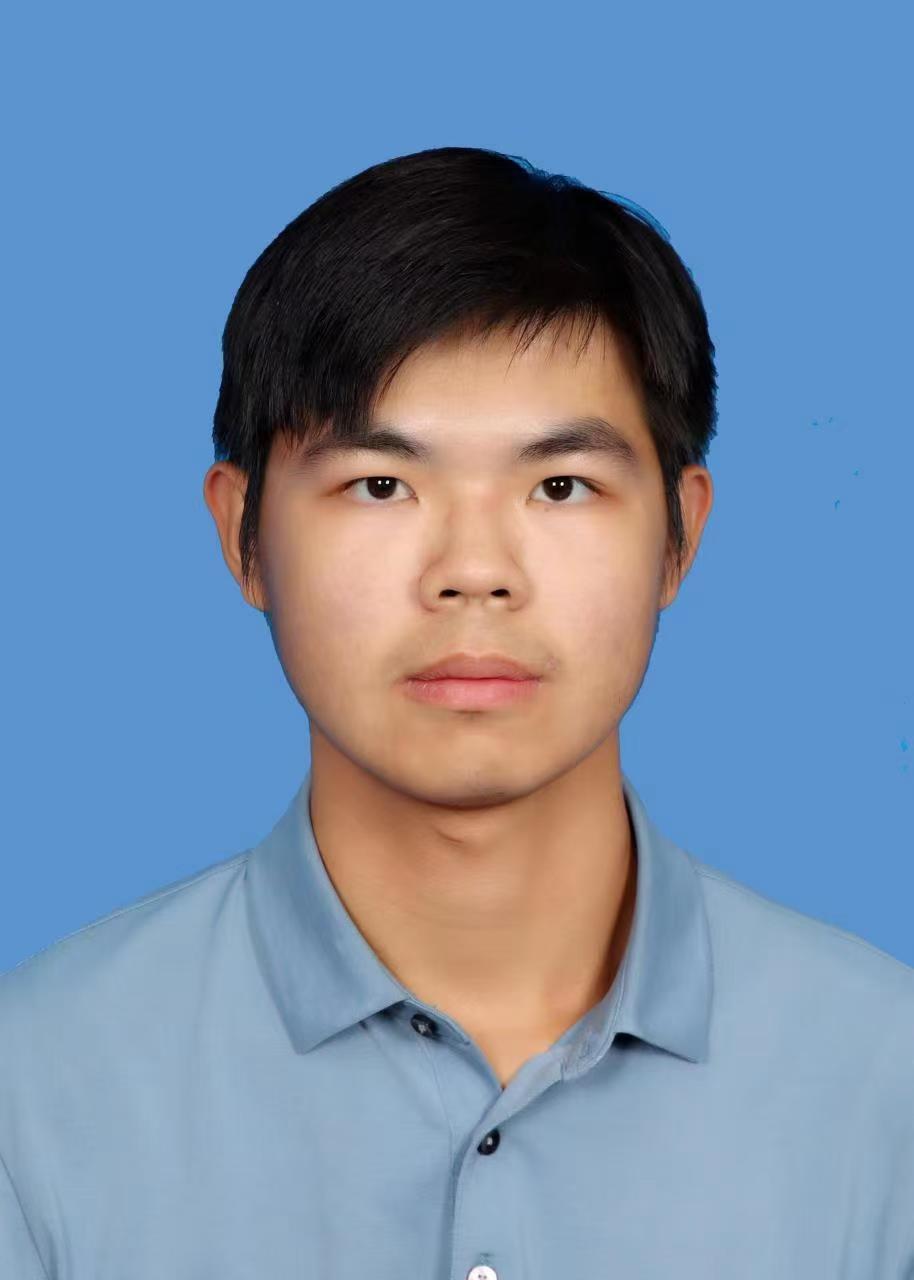}}]
{Yunlin Liu} is with the School of Computer Science and Technology, Guangdong University of Technology, Guangzhou, China. His current research interests include unsupervised machine learning and differential privacy.
\end{IEEEbiography}

\end{document}